%% file: main.tex
\documentclass{mystyle}

\usepackage{amssymb}
\usepackage{amsthm}
\usepackage{mathtools}
\usepackage{url}
\usepackage{multirow}
\usepackage{color}

\usepackage{siunitx} 
\usepackage{multirow} 
\usepackage{booktabs} 
\usepackage{longtable} 
\usepackage{rotating} 

\usepackage{threeparttable}

\usepackage{booktabs}
\usepackage{subcaption}
\usepackage{lipsum}

\begin{document}

\input{commands}

\title{Automatic White-Box Testing of First-Order Logic Ontologies}

\author{{\sc Javier \'{A}lvez, Montserrat Hermo, Paqui Lucio, German Rigau}\\[2pt]
Facultad de Inform\'atica, University of the Basque Country UPV/EHU, \\
Paseo Manuel de Lardizabal, 1, 20018-San Sebasti\'an, Spain.\\
}
\pagestyle{headings}
\markboth{J. \'ALVEZ, M. HERMO, P. LUCIO, G: RIGAU}{\sc Automatic White-Box Testing of First-Order Logic Ontologies}
\maketitle

\begin{abstract}
{
Formal ontologies are axiomatizations in a logic-based formalism. The development of formal ontologies is generating considerable research on the use of automated reasoning techniques and tools that help in ontology engineering. One of the main aims is to refine and to improve axiomatizations for enabling automated reasoning tools to efficiently infer reliable information. Defects in the axiomatization can not only cause wrong inferences, but can also hinder the inference of expected information, either by increasing the computational cost of, or even preventing, the inference. \\
In this paper, we introduce a novel, fully automatic white-box testing framework for first-order logic (FOL) ontologies.  Our methodology is based on the detection of inference-based redundancies in the given axiomatization. The application of the proposed testing method is fully automatic since a) the automated generation of tests is guided only by the syntax of axioms and b) the evaluation of tests is performed by automated theorem provers. Our proposal enables the detection of defects and serves to certify the grade of suitability --for reasoning purposes-- of every axiom. We formally define the set of tests that are (automatically) generated from any axiom and prove that every test is logically related to redundancies in the axiom from which the test has been generated. We have implemented our method and used this implementation to automatically detect several non-trivial defects that were hidden in various first-order logic ontologies. Throughout the paper we provide illustrative examples of these defects, explain how they were found, and how each proof --given by an automated theorem-prover--provides useful hints on the nature of each defect.  Additionally, by correcting all the detected defects, we have obtained an improved version of one of the tested ontologies: \ADIMENSUMO.
}
{Automated Theorem Proving, Fisrt-Order Logic, Ontology, Knowledge representation, White-box testing.}
\end{abstract}


\section{Introduction} \label{section:introduction}

Formal ontology development \cite{noy2001ontology,GuW04,Gru09,StS09} is a discipline whose goal is to represent explicit formal specifications (axiomatizations) of terms in a domain and relations between them. 
Many research areas --such as Semantic Web,  Knowledge Representation,  Commonsense Representation and Reasoning \cite{mccarthy1989artificial,minsky2007emotion,DaM15}--
have converged on the adoption of formal ontologies as explicit conceptualizations that are able to support automated reasoning. 
In this paper we focus on First-Order Logic (FOL) ontologies. 
Description Logic (DL) \cite{BCM10} is a family of formal knowledge representation languages that is very commonly used in the ontology area, but DL is less expressive than FOL. Ontologies such as e.g. CYC \cite{Matuszek+'06}, \DOLCE{} \cite{GGM03} and \SUMO{} \cite{Niles+Pease'01} which are  very useful in some areas --e.g. commonsense reasoning and natural language processing-- use more expressive languages (FOL or higher). 
Moreover, DL is a sublanguage of FOL and consequently the techniques presented in this paper can also be applied to DL ontologies.\footnote{Indeed, in this paper we report on the application of our techniques to the DL ontology DOLCE-CASL.}

As with other software artifacts, ontologies have to fulfill some previously specified requirements. Clear (and explicit) ontological distinctions and principles such as those provided by OntoClean \cite{GuW04} reduce the risk of classification mistakes in the ontology development process, and can simplify its maintenance. Both the creation of ontologies and the verification of its requirements are usually semi-automatic tasks that require a significant amount of human effort \cite{ALR12}. Once the ontology is stable and consistent (or at least no inconsistencies can be found by automatic means), the methods for ontology testing can be classified into two main categories: {\it black-box} and {\it white-box}, as in the field of software testing \cite{MSB12}. Black-box methods are based on the use of tests defined according to the requirements of the given ontology, while white-box methods are characterized by the fact that tests are created upon the particular specification/codification of the knowledge. 
In this paper, every ontology to be tested is considered to be the specification or codification, and a set of tests is automatically constructed from the axioms of the considered ontology. The construction of each test depends only on the syntactical form of an axiom. In this sense, our strategy is white-box.
As far as we know, the only testing method (for FOL ontologies) that can be classified as white-box is proposed in \cite{SSU17}, where the authors propose to create potentially unsatisfiable subsets of axioms by applying SInE strategies \cite{HoV11} on the basis of a selected {\it seed} symbol. In this way, large first-order knowledge bases can be proved to be inconsistent if some inconsistent subset is found. Among black-box testing methods, most frequent ones are based on consistency checking.
State-of-the-art reasoners such as FaCT++ \cite{TsH06}, Pellet \cite{SPC07} or HermiT \cite{GHM14} enable consistency proving in the case of DL ontologies. On the contrary, proving the consistency of ontologies expressed in FOL such as \SUMO{}  or \DOLCE{}  is much harder. Additionally, ontology testing methods include those based on the use of {\it competency questions} (CQs) \cite{GrF95} for validating functional requirements. That is, the {\it competency} of an ontology is described by means of a set of goals or problems that are expected to be answers according to its requirements. The process of obtaining CQs is not automatic but creative \cite{FGS13}. Depending on the size and complexity of the ontology, creating a suitable set of CQs is by itself a very challenging and costly task. In \cite{ALR15,ALR17}, we propose a method for the semi-automatic creation of CQs for \SUMO{}-based FOL ontologies on the basis of \WORDNET{} \cite{Fellbaum'98} and its mapping into \SUMO{} \cite{Niles+Pease'03}. Alternatively, in \cite{BFS13} the authors propose a tool that creates and processes CQs written in natural language for OWL ontologies. Finally, some other testing methods are based on cross-checking ontologies against another knowledge bases. For example, the authors of \cite{PaG15} propose to find errors in \DBPEDIA{} \cite{auer2007,bizer2009} by its alignment to \DOLCE{}, and contextual knowledge extracted from \DBPEDIA{} is used for detecting hidden errors in DL ontologies as described in \cite{TZN18}.

Regarding ontology debugging methods, the classification is slightly different \cite{PSK05}: {\it black-box} methods use reasoners as the oracle for a certain set of questions e.g., subsumption, satisfiability, etc.; {\it glass-box} methods are based on information extracted from the internals of reasoners, which are sometimes specifically adapted for the debugging task. There exists a large variety of techniques which are used in both classes of methods for DL ontologies. Among others, justification based techniques in black-box methods \cite{FrS05,JQH09,Hor11,SFF12}, axiom pinpointing in black-box \cite{BaS08} and glass-box methods \cite{BaP10}, and checking unsatisfiable dependant paths in glass-box methods \cite{ZOY17}. Further, black- and glass-box methods are combined in some proposals, using both axiom pinpointing \cite{SHC07} and justification based techniques \cite{KPH07}.

In \cite{GrF95}, the authors propose a methodology for the design and evaluation of ontologies on the basis of a set of CQs. The only requirement for applying the proposed methodology is the existence of a decision algorithm for the underlying logic. An adaptation of this methodology to FOL ontologies is introduced in \cite{ALR15}, which enables ontologies to be automatically evaluated by means of the use of FOL Automated Theorem Provers (ATPs). In this adaptation, the set of CQs is partitioned into two sets: {\it truth-tests} and {\it falsity-tests}, depending on whether one expects the conjecture to be entailed by the ontology ({\it truth-tests}) or not ({\it falsity-tests}). 
An example of truth-test is {\it ``Siblings have the same mother''}, since it is expected to be entailed by an ontology (Of course, the ontology should axiomatize the involved concepts).
This truth-test belongs to the {\it Commonsense Reasoning} CSR domain of the {\it Thousands of Problems for Theorem Provers} (TPTP) problem library\footnote{\url{http://www.tptp.org}} \cite{Sut09} and the latter one is a CQ in the benchmark proposed in \cite{ALR15}.
On the contrary, the conjecture {\it ``Some herbivores eat animals''} is a falsity-test since it is not expected to be entailed by the ontology, in spite that the ontology should axiomatize the involved concepts.

In this paper, we introduce a completely automatic methodology for the evaluation of FOL ontologies by means of a set of automatically generated falsity-tests and truth-tests.
We define the logical foundations of these sets of tests and prove the correctness of the methodology. 
We also describe the application of our methodology to \DOLCE{} \cite{Gangemi+'02}, \KEPLER{} (formal proof of the Kepler conjecture) \cite{HHM10} and \ADIMENSUMO{} \cite{ALR12}. In particular, we review in detail the kind of defects that have been detected in those ontologies.
An improved version of \ADIMENSUMO{} v2.6 has been obtained by correcting all the defects detected following the introduced methodology. 
An example of incorrect axiom that we have detect in \ADIMENSUMO{} v2.4 following our methodology is part of the axiomatization of the relation {\it sibling}. The wrong axiom asserts that {\it ``Any two members of the same broad are siblings"}, instead of ``Any two {\it different} members of the same broad are siblings" (see Subsection \ref{subsec:incorrect-axioms} for details). The latter is the correct version of the axiom in \ADIMENSUMO{} v2.6.

The paper is organized as follows. First, we briefly describe the ontologies \DOLCE{}, \KEPLER{} and \ADIMENSUMO{} in Section \ref{section:FOL-Ontologies}.
In Section \ref{section:methodology}, we introduce the proposed methodology for evaluating ontologies utilizing ATPs. Next, in Section \ref{section:tests}, we formally define the set of tests that are proposed for a given axiom. Then, in Section \ref{section:example}, we provide a detailed example illustrating the calculation of tests. In Section \ref{section:correctness}, we prove the correctness of the proposed set of tests. In Section \ref{section:defects}, we report on the main kinds of defects that we have found in the evaluated ontologies, explaining specific examples of four different types.
We provide a summary of our experimental results in Section \ref{section:experimentation}. Finally, we give some conclusions and discuss future work.

\section{FOL Ontologies} \label{section:FOL-Ontologies}

FOL formulas are constructed on an alphabet (or signature) of function and predicate symbols, using the logical connectives of negation ($\neg$), conjunction ($\wedge$), disjunction ($\vee$), implication ($\rightarrow$) and double-implication ($\leftrightarrow$) , as well as the universal and existential quantifiers (resp. $\forall$ and $\exists$).
The notation $\quantifier{ x}$ stands for a sequence of quantifiers (on variables) $Q_1  x_1 \ldots Q_n  x_n$ such that $n \geq 0$ and $Q_i \in \{ \exists, \forall \}$ for each $1 \leq i \leq n$. 
In the FOL formulas in this paper, functions symbols (in particular, constants) start with a capital letter, whereas predicates start with a lower-case and variables are lower-case letters (possibly with subindices).
We use {\it sentence} to refer to any FOL formula for which all its variable occurrences are in the scope of (or bound by) a quantifier. We assume that the reader has some familiarity with the syntax and basic notions of FOL. 

FOL ontologies consist of a set of FOL sentences called {\it axioms}. Typically, ontology axioms are classified into {\it rules} and {\it non-rules}. Rules are universally closed implications. An example of rule axiom in \DOLCE{} is (\ref{formula:NonEmptyUniversalKIFDolce}).
It is also common to call the left-hand part of the implication {\it antecedent}, and the right-hand part {\it consequent}.
In this section, we introduce the main features of the FOL ontologies that have been used for the evaluation of our white-box testing strategy ---\DOLCE{}, \KEPLER{} and \ADIMENSUMO{}--- and provide some figures about their content.
We also provide some examples of axioms and tests in the form of FOL sentences. Each FOL ontology uses its own alphabet of functions and predicates, named with strings (e.g. $world$, $mother$, $agent$, ...) that usually try to express its meaning,
but their formal meaning is given by their FOL axiomatization in the ontology.

\DOLCE{} ({\it Descriptive Ontology for Linguistic and Cognitive Engineering}) is the first-module of a {\it Library of Foundational Ontologies} being developed within the WonderWeb project \cite{Gangemi+'02}. Its main purpose is supporting effective cooperation between multiple artificial agents and establishing consensus in a mixed society where artificial agents cooperate with human beings. The partial mappings from \WORDNET{} \cite{Fellbaum'98} to \DOLCE{} \cite{GGM03} enable the connection of \DOLCE{} to other semantic resources such as the Multilingual Central Repository (MCR) \cite{ARV04} and its application to advanced Natural Language Processing, Knowledge Engineering and Semantic Web tasks \cite{VAR13}. The domain of discourse of \DOLCE{} is restricted to the notion of {\it particulars} ---entities which have no instances. Similarly, particulars are characterized and organized around a taxonomy of 37 {\it universals}---entities that can have instances--- and universals are organized around the notion of {\it world}. However, no particular and no world is explicitly defined. \DOLCE{} is originally expressed in KIF  according to the standard proposed in \cite{KIF98}\footnote{\url{http://logic.stanford.edu/kif/dpans.html}} and simplified translations into various logical languages have been proposed (\DOLCE{}-Lite-Plus\footnote{\url{http://www.loa.istc.cnr.it/old/DOLCE.html}}). The KIF version of \DOLCE{} uses row variables ---which produces variable-arity relations--- and quantified predicate symbols. Hence, we have applied the translation described in \cite{ALR12} for its transformation into a pure FOL formula (from now on, \KIFDOLCE{}). As result, we have obtained 257 rule-axioms such as the following:
\begin{equation} \label{formula:NonEmptyUniversalKIFDolce}
\forall w \; \forall f \; ( \; ( universal(f) \wedge world(w) ) \to nep(w,f) \; )
\end{equation}
where \textPredicate{nep}$(w,f)$ stands for the non-emptiness of the universal $f$ in the world $w$. Vampire v4.1 \cite{RiV02,KoV13} proves that \KIFDOLCE{} is consistent. In addition, a simplified translation of \DOLCE{} into CASL \cite{ABK02} (from now on, \CASLDOLCE{}) is available in Hets \cite{MML07} and its consistency is proved in \cite{KuM11}. This translation consists of 416 non-atomic formulas such as
\begin{equation} \label{formula:NonEmptyEndurantCASLDolce}
\exists y \; ( \; pED(y) \; )
\end{equation}
where \textPredicate{pED} is used to state that the universal \textPredicate{Endurant} has some particular.\footnote{37 formulas like (\ref{formula:NonEmptyEndurantCASLDolce}) are used in \CASLDOLCE{} for stating the property in formula (\ref{formula:NonEmptyUniversalKIFDolce}).} Both FOL versions of \DOLCE{} ---\KIFDOLCE{} and \CASLDOLCE{}--- were expected to be non-defective due to their reduced size and their mature state of development.

\KEPLER{} ({\it formal proof of the Kepler conjecture}) is an ontology that has been derived from the Flyspeck project \cite{HHM10} for its use in the {\it CADE ATP System Competition CASC-J8} \cite{Sut16}. The purpose of the Flyspeck project is to give a formal proof of the Kepler conjecture, which asserts that no packing of congruent balls in three-dimensional Euclidean space has a density greater than that of the face-centered cubic packing \cite{HaF06}. Its participants claim that it ``is the most complex formal proof ever undertaken'' and estimate that it may take about twenty working-years to complete the formalization. To this end, every logical inference is checked against the foundational axioms of mathematics with the help of a computer and, no matter how trivial they are, no step is skipped. We have used the version of \KEPLER{} that was provided for the CASC competition, since it already consists in a pure FOL axiomatization with 78,500 axioms, such as the following:
\begin{displaymath}
\forall a \; \forall x \; \forall y \; ( \; s(a,x) = s(a,y) \; \to \; s(a,y) = s(a,x) \; )
\end{displaymath}
Though \KEPLER{} is much larger than \DOLCE{}, \KEPLER{} has also reached a very mature state of development. Consequently, we did not expect to discover many defects in \KEPLER{} either.

Finally, \ADIMENSUMO{} has been derived from \SUMO{} ({\it Suggested Upper Merged Ontology})\footnote{\url{http://www.ontologyportal.org}} \cite{Niles+Pease'01}, which was promoted by a group of engineers from the IEEE Standard Upper Ontology Working Group as a formal ontology standard during the nineties of the past century. Their goal was to develop a standard upper ontology to promote data interoperability, information search and retrieval, automated inference and natural language processing. \SUMO{} is expressed in SUO-KIF ({\it Standard Upper Ontology Knowledge Interchange Format} \cite{Pea09}), which is a dialect of KIF, and its syntax goes beyond FOL. Consequently, \SUMO{} cannot be directly used by FOL ATPs without a suitable transformation.
Furthermore, in order to support higher-order aspects, a translation of \SUMO{} is also required for its use by means of pure higher-order theorem provers \cite{PeB13}.
In \cite{ALR12}, the authors use ATPs for reengineering around 88\% of the top and the middle levels of \SUMO{} into \ADIMENSUMO{},\footnote{The first version of \ADIMENSUMO{} is v2.2.} which can be expressed as a FOL formula. This translation is based on a small set of meta-predicates ---and its axiomatization--- that are required to define the knowledge of \SUMO{} according to its organization around four kinds of concepts: {\it objects}, {\it classes}, {\it relations} and {\it properties}. Some of these meta-predicates are \textPredicate{\$instance}, \textPredicate{\$subclass}, \textPredicate{\$disjoint} and \textPredicate{\$partition}.\footnote{In \ADIMENSUMO{}, meta-predicates names are marked with a $\$$. In this paper, we omit these marks since ATPs deal with them like any other predicate symbol in the alphabet.}  
In \ADIMENSUMO{}, like in other ontologies expressed in KIF (e.g. \SUMO{} or \KIFDOLCE{}), $instance$ is used to assert that an object is in a class. For example, the atom $instance(h,Herbivore)$, instead of $Herbivore(h)$, is used to express that object $h$ is in the class $Herbivore$.
In addition, \cite{ALR12} provides a suitable translation of domain (or type) information of relations which, in \SUMO{}, is (separately) provided by means of \textPredicate{domain} axioms. For example, in \ADIMENSUMO{}, there are four non-rule axioms asserting that the first argument of predicate \textPredicate{\$instance} is an object, whereas the second one is a class, and that both arguments of \textPredicate{\$subclass} are classes. \ADIMENSUMO{} has also three rules for axiomatizing \textPredicate{\$subclass} as a partial order, i.e. each rule respectively says that \textPredicate{\$subclass} is reflexive, antisymmetric and transitive. Additionally, the following rule axiomatizes \textPredicate{\$instance} in terms of \textPredicate{\$subclass}:
\begin{displaymath}
\forall x \; \forall y \; \forall z \; ( \; (\;\$instance(x,y)\;\wedge\; \$subclass(y,z)\;) \; \to \; \$instance(x,z) \; ).
\end{displaymath}
Many other meta-predicates are defined in terms of \textPredicate{\$subclass} and \textPredicate{\$instance}. For example, the predicate \textPredicate{\$disjoint} is defined by the following rule in  \ADIMENSUMO{}:
\begin{displaymath}
\forall x \; \forall y \; ( \; \$disjoint(x,y)\; \leftrightarrow \;\forall z\; (\;\neg\$instance(z,x)\;\vee\; \neg\$instance(z,y)\;) \; ).
\end{displaymath}
The interested reader is referred to \cite{ALR12} for  a detailed description of the axiomatization and the translation.

\ADIMENSUMO{} ---like \DOLCE{}--- also uses row variables and quantified predicate symbols. In \cite{ALR15}, we introduce an evolved version of \ADIMENSUMO{} (namely, v2.4) and demonstrate its inference capabilities in practice. More specificaly, we exploit the whole mapping of \WORDNET{} to \SUMO{} \cite{Niles+Pease'03} in order to obtain a set of CQs by following a black-box testing strategy. As reported in \cite{ALR15}, we have experimentally tested \ADIMENSUMO{} v2.4 using the resulting set of CQs and no defect has been detected. Additionally, we have used \ADIMENSUMO{} v2.4 and the same set of CQs for an experimental comparison of several FOL ATPs in \cite{ALR16}. The state of development of \ADIMENSUMO{} v2.4 was not mature and we were able to detect various defects by following the white-box testing methodology introduced in this paper. As result of correcting all the detected defects, we obtained \ADIMENSUMO{} v2.6, which has been already used in the experimentation reported in \cite{AGR18,AlR18}.

\begin{table}[h]
\caption{ Some figures about the evaluated ontologies}
\label{table:FOLOntologiesFigures}
\centering
\begin{tabular}{lrrr}
\hline \\[-4mm]
{Ontology} & {Non-rules} & {Rules} & {Total} \\[1pt]
\hline \\[-3mm]
{\KIFDOLCE{}} & 0 & 257 & 257 \\
{\CASLDOLCE{}} & 0 & 416 & 416 \\
{\KEPLER{}} & 275 & 78,225 & 78,500 \\
{\ADIMENSUMO{} v2.4} & 4,635 & 2,785 & 7,420 \\
{\ADIMENSUMO{} v2.6} & 4,638 & 2,799 & 7,432 \\
\hline
\end{tabular}
\end{table}

In Table \ref{table:FOLOntologiesFigures}, we summarize some figures about \DOLCE{}, \KEPLER{} and \ADIMENSUMO{} (v2.4 and v2.6): the number of (non-rule and  rule) axioms that result from their transformation into a pure FOL formula (no transformation is required for \KEPLER{}).

\section{Automatic Testing of FOL Ontologies} \label{section:methodology}

In this section, we describe the framework and methodology proposed in \cite{ALR15} for the evaluation of FOL ontologies using an existing set of conjectures consisting of falsity-tests and truth-tests. 

A conjecture is decided to be entailed by the ontology only if the ATP is able to find a proof within the provided execution-time limit. The proof (which can be reported by the ATP) of a conjecture that is not-expected to be proved (i.e. a falsity-test) provides hints on the defects of the axiomatization. 
In general, when a proof is not found, the ATP can report, either that the conjecture is not entailed, or that the time limit was reached. In the first case, the ATP could produce a countermodel showing that the conjecture is not satisfied in a specific model of the ontology. Countermodels of truth-tests --similarly proofs of falsity-tests-- could be used as hints for detecting defects in the ontology. For large ontologies, axiom selection methods \cite{HoV11} can be used to increase the number of useful answers (countermodels or proofs), i.e. to reduce the number of tests that exceeds the time limit. 
However, this approach is beyond the scope of the methodology presented in this paper. In fact, this could be a future improvement of our framework. 
Consequently, if ATPs find a proof, then {\it truth-tests} and {\it falsity-tests} are classified as {\it proved}. Otherwise, if no proof is found, then we classify both {\it truth-} and {\it falsity-tests} as {\it unknown} because we do not know whether the corresponding conjectures are entailed or not. 
For example, the truth-test {\it ``Siblings have the same mother''} is given by the following FOL sentence:
\begin{eqnarray} \label{formula:SiblingsMother}
\forall o1 \; \forall o2 \; \forall o3 \; ( \; ( mother(o1,o2) \wedge sibling(o1,o3)\; ) \to mother(o3,o2) \; ) 
\end{eqnarray} 
and it is easily proved by ATPs to be entailed by \ADIMENSUMO{} v2.6. Thus, the truth-test (\ref{formula:SiblingsMother}) is classified as {\it proved}.
On the contrary, for the falsity-test {\it ``Some herbivores eat animals''}, that is given by the FOL sentence:
\begin{eqnarray} \label{formula:Herbivore}
\exists h \; \exists a \; \exists e \; (& instance(h,Herbivore) \wedge instance(a,Animal)  \hspace{1,3cm} & \\ \nonumber
&  \wedge \;instance(e,Eating)  \wedge agent(e,h) \wedge patient(e,a)  \; ) &
\end{eqnarray}
is classified as {\it unknown} in \ADIMENSUMO{} v2.6.

It is worth noting that truth-tests classified as proved will be used to grade the suitability of the axiom they come from, whereas falsity-test classified as proved  will be used to detect defects in the axiomatization usually with the help of the proof reported by the ATP.
As a consequence, both,
the grade of suitability of a formula, and the detection on defects, rely on the ATPs utilized and also on the parameter configuration set. 

In our proposal, the set of conjectures is automatically constructed by following white-box testing strategies. We create two sets\footnote{They are formally defined in Definitions \ref{defn:FalsityTests} and \ref{defn:TruthTests}, respectively.} of conjectures: $FT(\phi)$ and $TT(\phi)$, for each axiom $\phi$ in the tested ontology, and this construction depends only on the syntactical form of the axiom $\phi$.  The purpose of  $FT(\phi)$ is to detect {\it defects} in the axiomatization. Each conjecture $\alpha$ in $FT(\phi)$ is called a {\it falsity-test} because $\alpha$ is not expected to be entailed by the ontology. If some $\alpha\in FT(\phi)$  is inferred from the ontology, then $\phi$ contains some {\em redundant} subformula. For example, we detect that the following axiom, extracted from \ADIMENSUMO{} \cite{ALR12}, is defective:
\begin{equation} \label{formula:CenterOfCircle}
\forall c \; ( \; instance(c,Circle) \; \to \; \exists p \; ( CenterOfCircleFn(c) = p ) \; )
\end{equation}
by means of the falsity-test: 
\begin{equation} \label{goal:CenterOfCircle}
\forall c \; \exists p \; ( \; CenterOfCircleFn(c) = p \; )
\end{equation}
Conjecture (\ref{goal:CenterOfCircle}) is trivially proved since equality is reflexive. 
In other words, from the reflexivity axiom: 
\begin{equation*} \label{axiom:reflexivity}
\forall x \;  ( \; x = x \; )
\end{equation*}
it is easy to prove that
\begin{equation} \label{goal:CenterOfCircle2}
\forall c \; ( \; CenterOfCircleFn(c) = CenterOfCircleFn(c) \; )
\end{equation}
and then (\ref{goal:CenterOfCircle}) is easily entailed from (\ref{goal:CenterOfCircle2}).
We conclude that the antecedent (left-hand part of the implication) of axiom (\ref{formula:CenterOfCircle}) is redundant, i.e. the consequent (right-hand part of the implication) is entailed whatever its antecedent would be. That makes axiom (\ref{formula:CenterOfCircle}) to be redundant itself, since by removing the redundant antecedent in (\ref{formula:CenterOfCircle}), we get (\ref{goal:CenterOfCircle}) and, moreover, the latter is already entailed by the ontology. \\
Our notion of redundancy is related to its practical use in the reasoning process, and it is formally introduced in Definition \ref{defn:redundancy1}. However, sometimes redundancy is caused by a different kind of defect ranging from typos (e.g. a misspelled or misplaced variable symbol) to incorrect axioms (e.g. a necessary condition is not required). In particular, our proposal enables defects to be detected that we have classified in the following four classes: {\it typos}, {\it redundant axioms}, {\it redundant subformulas (in axioms)} and {\it incorrect (inaccurate) axioms}.

The set $TT(\phi)$ consists of the negation of all the conjectures in $FT(\phi)$. They are called {\em truth-tests} because they are expected to be inferred and they are used to  grade the {\it suitability} of axioms for reasoning purposes. 
That is, whenever no conjecture in $FT(\phi)$ is classified as proved for the axiom $\phi$, we identify three different grades of suitability according to the conjectures in $TT(\phi)$ that are entailed by the ontology. Axiom $\phi$ is {\it completely suitable} if the ontology entails all the truth-tests in $TT(\phi)$. Otherwise $\phi$ is {\it partially suitable} if at least one conjecture in $TT(\phi)$ is proved; and $\phi$ is {\it unsuitable} if no one conjecture in $TT(\phi)$ is proved.
For example, the following axiom obtained from \ADIMENSUMO{}:
\begin{equation} \label{formula:DrivingPatient}
\forall d \; ( \; instance(d,Driving) \; \to \\
 \exists v \; ( \; instance(v,Vehicle) \wedge patient(d,v ) \; ) \; )
\end{equation}
is classified as completely suitable by means of an automatically generated set of eight different truth-tests used as conjectures by the ATPs.\\

\section{Automatic Generation of Tests} \label{section:tests}

In this section, we introduce the definition of  two functions $FT$ and $TT$ that respectively compute the sets of falsity- and truth-tests for a given axiom. These two functions are defined by (structural) induction on the syntactic structure of the input axiom, hence our automatic test generation is syntax-based.

We use lower-case Greek letters for arbitrary formulas and capital Greek letters (e.g. $\Phi$) for (finite) sets of sentences, that equivalently can be seen as conjunctions of all their members. The notation $\phi[\alpha]$ represents the formula $\phi$ and, at the same time, means that $\alpha$ is a subformula of $\phi$. We denote by $\phi[\alpha/\gamma]$ the formula that results from replacing every occurrence of $\alpha$ (as a subformula of $\phi$) with $\gamma$.
Given any FOL formula $\phi$, the expressions $( \phi )^\exists$ and $( \phi )^\forall$ respectively, denote the existential and universal closure of $\phi$. Note that if $\phi$ is a sentence, then $( \phi )^\exists$ and $( \phi )^\forall$ are identical to $\phi$.  Two formulas $\phi$ and $\psi$ are semantically (or logically) {\it equivalent}, in symbols $\phi \equiv \psi$, if and only if they have exactly the same models. Given any set of sentences $\Phi$, we say that two formulas $\alpha$ and $\beta$ are {\it $\Phi$-equivalent} if and only if $\Phi\models( \alpha \leftrightarrow \beta )^\forall$ (i.e. every model of $\Phi$ is also a model of $( \alpha \leftrightarrow \beta )^\forall$). Note that if $\alpha$ and $\beta$ are sentences, then $(\alpha \leftrightarrow\beta)^\forall$ and $\alpha \leftrightarrow\beta$ are the same formula.

Most of the axioms in an ontology are (universally closed) implications, e.g. see axioms (\ref{formula:CenterOfCircle}) and (\ref{formula:DrivingPatient}) above.  
Redundancy in the subformulas of axioms can be detected by proving unexpected conjectures.
For example, an axiom 
$( \gamma \to \psi )^\forall$ is redundant in an ontology $\Phi$, whenever 
$\Phi$ entails one of the conjectures $\psi^\forall$ or $(\neg\gamma)^\forall$. However, if we consider a formula $(\gamma \to (\phi\vee\psi))^\forall$ such that $\Phi$ entails the conjecture $(\neg\psi)^\forall$, then the subformula $\psi$ is redundant in the axiom, since  $(\gamma \to (\phi\vee\psi))^\forall$ is $\Phi$-equivalent to $(\gamma \to \phi)^\forall$.
Implication is an important connective in our test generation.
However, for the sake of a more uniform treatment and a clearer presentation, we use the equivalence $\psi\rightarrow\gamma\equiv\neg\psi\vee\gamma$ to transform implications into disjunctions. In this sense, a subformula $\alpha \vee \beta$ can be seen as the result of transforming (and simplifying) the implication $\neg\alpha \rightarrow \beta$ or the implication $\neg\beta \rightarrow \alpha$.
Consequently, given a formula $\phi[\alpha \vee \beta]$ (where the subformula $\alpha \vee \beta$ is not in the scope of  negation)\footnote{In the formal definition, $\phi$ is in negation normal form (see explanations just above Definition \ref{defn:FalsityTests}).}, we propose the following falsity-tests for the detection of defects by searching redundant subformulas:  $\alpha^\forall$, $\beta^\forall$, $( \neg \alpha )^\forall$ and $( \neg \beta )^\forall$. Roughly speaking, when  $\alpha^\forall$ is entailed by $\Phi$, we detect that $\beta$ is redundant. Symmetrically, the entailment of  $( \beta )^\forall$ makes $\alpha$ redundant. Whenever $( \neg \alpha )^\forall$ or $( \neg \beta )^\forall$ are entailed, then $\alpha$ or $\beta$ respectively, are redundant.

Next, we illustrate the idea for generating falsity-tests by means of an example.

\begin{example} \label{ex:sibling-FT}
Consider the following set of four axioms of \SUMO:
\begin{equation}
\label{axiom:siblingIrreflexiveRelation} 
instance(sibling, IrreflexiveRelation)
\end{equation}
\begin{equation}
\label{axiom:domainSibling1}
domain(sibling,1,Organism)
\end{equation}
\begin{equation}
\label{axiom:domainSibling2} 
domain(sibling,2,Organism)
\end{equation}
\begin{eqnarray}
 \nonumber
\forall m_1\;\forall m_2\; \forall b\; 
(\; 
\hspace{-5mm} & ( & \hspace{-5mm} instance(b,Brood) \wedge 
member(m_1,b) \wedge
member(m_2,b)
) \\
& \;\rightarrow \; & \label{axiom:BroodSibling}
sibling(m_1,m_2)       
\;)
\end{eqnarray}
According to the type information in axioms (\ref{axiom:domainSibling1}-\ref{axiom:domainSibling2}), both arguments of \textPredicate{sibling} are restricted to be instance of \textConstant{Organism}.  Consequently, by translation (see \cite{ALR12} for more details), axiom (\ref{axiom:BroodSibling}) gives raise to the following rule-axiom in \ADIMENSUMO{} v2.4:  
\begin{equation} \label{formula:oneSortedBroodSibling}
\begin{array}{rll}
\forall  m_1 \; \forall  m_2 \; \forall  b \;( \hspace{-7pt} &    ( \hspace{-10pt} &  instance( m_1,Organism) \; \wedge \\
 & &  instance( m_2,Organism) \; \wedge \\
 & & instance( b,Brood) \; \wedge \\
 & & member( m_1, b) \; \wedge \\
 & & member( m_2, b)  \; ) \; \rightarrow\\
 & \multicolumn{2}{l}{sibling( m_1, m_2) \; )}
\end{array}
\end{equation}
Since its inner subformula is logically equivalent to the following disjunction (in negation normal form):
\begin{equation} \label{subformula:oneSortedBroodSiblingDisjunction}
\begin{array}{rl}
( \;\hspace{-10pt} & 
 \neg instance( m_1,Organism) \; \vee\\
 & \neg instance( m_2,Organism) \; \vee \\
 & \neg instance( b,Brood) \; \vee\\
 & \neg member( m_1, b) \; \vee\\
 & \neg member( m_2, b) \; )
\end{array}
\begin{array}{c}
~ \\
~ \\
~ \\
~ \\
\vee
\end{array}
\begin{array}{l}
~ \\
~ \\
~ \\
~ \\
sibling( m_1, m_2)
\end{array}
\end{equation}
The following four conjectures belong to the set of falsity-tests for axiom (\ref{formula:oneSortedBroodSibling}):

\begin{equation}
 \label{goal:oneSortedBroodSiblingDisjunctionAlphaFalsityTest}
\begin{array}{rl}
\forall  m_1 \; \forall  m_2  \; \forall  b \; ( \hspace{-10pt} & \neg  instance( m_1,Organism) \; \vee \\
 & \neg instance( m_2,Organism) \; \vee \\
 & \neg instance( b,Brood) \; \vee \\
 & \neg member( m_1, b) \; \vee \; \neg member( m_2, b) )
\end{array}
\end{equation} 
\begin{equation} \label{goal:oneSortedBroodSiblingDisjunctionBetaFalsityTest}
\forall  m_1 \; \forall  m_2 \; ( sibling( m_1, m_2) ) \end{equation} 
\begin{equation} \label{goal:oneSortedBroodSiblingDisjunctionNegAlphaFalsityTest}
\begin{array}{rl}
\forall  m_1 \; \forall  m_2  \; \forall  b \; ( \hspace{-10pt} &  instance( m_1,Organism) \; \wedge \\
 &  instance( m_2,Organism) \; \wedge \\
 &  instance( b,Brood) \; \wedge \\
 &  member( m_1, b) \; \wedge \;  member( m_2, b) )
\end{array}
\end{equation} 
\begin{equation} \label{goal:oneSortedBroodSiblingDisjunctionNegBetaFalsityTest}
\forall  m_1 \; \forall  m_2 \; ( \neg \; sibling( m_1, m_2) )
\end{equation}
Using the methodology described in Section \ref{section:methodology}, falsity-test (\ref{goal:oneSortedBroodSiblingDisjunctionAlphaFalsityTest}) is classified as proved, while the remaining falsity-tests  (\ref{goal:oneSortedBroodSiblingDisjunctionBetaFalsityTest}-\ref{goal:oneSortedBroodSiblingDisjunctionNegBetaFalsityTest}) are classified as unknown. Thus, falsity-test (\ref{goal:oneSortedBroodSiblingDisjunctionAlphaFalsityTest}) enables a defect to be detected as explained in Subsection \ref{subsec:incorrect-axioms}. In fact, axiom (\ref{axiom:BroodSibling}) has been corrected in \ADIMENSUMO{} v2.6.
\end{example}

The function $FT$ is to be applied to every ontology axiom. Therefore, $FT$ is defined by structural induction on a formula $\phi$ that belongs to the language
\begin{displaymath}
\phi ::= \ell \mid \phi \vee \phi \mid \phi \wedge \phi \mid \forall x \phi \mid \exists x \phi
\end{displaymath}
where $\ell$ stands for literal (atom or negated atom). In addition, we can suppose that any quantifier in $\phi$ has a different variable symbol. Our assumption is not a limitation since it is well-known that any FOL formula can be transformed into a logically equivalent one in the above language (see e.g. \cite{Fit90}). The transformation follows the three initial steps of the standard algorithm that transforms any FOL formula into its conjunctive normal form:
\begin{enumerate}
\item Rectification or elimination of variable clashing: rename clashing variables so that each quantifier has a unique variable symbol.
\item Transformation into {\it arrow-free form}: repeatedly apply the following two logical equivalences left-to-right until none can be applied: 
\begin{itemize}
\item $\psi \leftrightarrow \gamma \equiv ( \psi\to \gamma ) \wedge ( \gamma \to \psi )$
\item $\psi \to \gamma \equiv ( \neg \psi ) \vee \gamma$
\end{itemize}
\item Transformation into {\it negation normal form}: repeatedly apply the following five logical equivalences (left-to-right) until none can be applied: 
\begin{itemize}
\item $\neg \neg \psi \equiv \psi$ 
\item $\neg ( \psi \wedge \gamma ) \equiv ( \neg \psi ) \vee ( \neg \gamma )$
\item $\neg ( \psi \vee \gamma ) \equiv ( \neg \psi ) \wedge ( \neg \gamma )$
\item $\neg \forall  x \; \psi \equiv \exists  x \; \neg \psi$
\item $\neg \exists  x \; \psi \equiv \forall  x \; \neg \psi$
\end{itemize}
\end{enumerate}
In what follows, we say that the formulas in the above language are in {\it arrow-free and negation normal form} (in abbreviated form,  {\it af-nnf}). For example, formula (\ref{subformula:oneSortedBroodSiblingDisjunction}) in Example \ref{ex:sibling-FT} is in af-nnf.

\begin{definition} \label{defn:FalsityTests}
For any af-nnf formula $\phi$, the function $FT$ is recursively defined as
\begin{displaymath}
\begin{array}{rcl}
\FT{\phi} & = &
\begin{cases}
\emptyset & \mbox{if } \phi \mbox{ is a literal} \\[-2mm]
\FTP{\phi} & \mbox{if } \phi = \alpha \vee \beta \\[-2mm]
\FT{\alpha} \; \cup 
 \FT{\beta} & \mbox{if } \phi = \alpha \wedge \beta \\[-2mm]
\FT{\alpha} & \mbox{if } \phi = \forall  x \; \alpha \mbox{ or } \phi = \exists  x \; \alpha
\end{cases}
\end{array}
\end{displaymath}
where the function $FT_0$ is defined as follows:
\begin{displaymath}
\begin{array}{rcl}
\FTP{\phi} & = &
\begin{cases}
\emptyset & \mbox{if } \phi \mbox{ is a literal} \\[-2mm]
\{ \; ( \alpha )^\forall, \; ( \beta )^\forall, \; ( \neg \alpha )^\forall, \; ( \neg \beta )^\forall \; \} \;\\[-4mm]
\hspace{25pt}  \cup \;\FTP{\alpha} \; \cup \; \FTP{\beta} & \mbox{if } \phi = \alpha \vee \beta \mbox{ or } \phi = \alpha \wedge \beta \\[-2mm]
\{ \; ( \alpha )^\forall, \; ( \neg \alpha )^\forall \; \} \; \cup \; \FTP{\alpha} & \mbox{if } \phi = \forall  x \; \alpha \mbox{ or }  \phi = \exists  x \; \alpha
\end{cases}
\end{array}
\end{displaymath}
\end{definition}

The idea behind Definition \ref{defn:FalsityTests} can be summed up as follows: for any falsity-tests $\{ ( \delta )^\forall, \; ( \neg \delta )^\forall \} \subseteq \FT{\phi}$, the sentence $\phi$ is logically equivalent to a (possibly quantified) conjunction of two formulas such that the first one is equal to a disjunction with $\quantifier{ y} \; \delta$ as subformula for some prefix of quantifiers $\quantifier{ y}$. 
Lemma \ref{lemma:subformula-FT} (in Section \ref{section:correctness}) formally states this idea and it is the key result for proving the correctness of our method (see Theorem \ref{thm:FT-correctness} in Section \ref{section:correctness}).

\begin{remark} \label{remark: n-ary}
For the sake of simplicity, in Definition \ref{defn:FalsityTests} we consider the binary connectives of
conjunction and disjunction, however we have implemented its natural generalization to n-ary connectives. For example, for $\phi = \alpha\wedge\beta\wedge\gamma$:
$$
\FTP{\phi}= \{ \; ( \alpha )^\forall, \; ( \beta )^\forall, \; ( \gamma )^\forall, \; ( \neg \alpha )^\forall, \; ( \neg \beta )^\forall \;  ( \neg \gamma )^\forall  \} \; \cup \FTP{\alpha} \; \cup \; \FTP{\beta} \cup \; \FTP{\gamma}.
$$
\end{remark}

It is obvious that the function $FT$ could produce many repeated tests. For example, for $\phi = \forall  x \; ( ( \exists  y \; \alpha ) \vee ( \forall  z \; \beta ) )$ we have that
\begin{eqnarray*}
\FT{\phi} & = & \FT{\forall  x \; ( ( \exists  y \; \alpha ) \vee ( \forall  z \; \beta ) )} \\
 & \stackrel{\forall}{=} & \FTP{( \exists  y \; \alpha ) \vee ( \forall  z \; \beta )} \\
 & \stackrel{\vee}{=} & \{ \; ( \exists  y \; \alpha )^\forall, ( \forall  z \; \beta )^\forall, ( \neg \exists  y \; \alpha )^\forall, ( \neg \forall  z \; \beta )^\forall \; \} \cup \\
 & & \hspace{25pt} \FTP{\exists  y \; \alpha} \cup \FTP{\forall  z \; \beta} \\
 & \stackrel{\exists,\forall}{=} & \{ \; ( \exists  y \; \alpha )^\forall, ( \forall  z \; \beta )^\forall, \; ( \neg \exists  y \; \alpha )^\forall, ( \neg \forall  z \; \beta )^\forall \; \} \cup \\
 & & \hspace{25pt} \{ \; ( \alpha )^\forall, ( \neg \alpha )^\forall \; \} \cup \FTP{\alpha} \cup \\
 & & \hspace{25pt} \{ \; ( \beta )^\forall, ( \neg \beta )^\forall \; \} \cup \FTP{\beta}
\end{eqnarray*}
where $( \forall  z \; \beta )^\forall$ and $( \beta )^\forall$ are both the same formula $\forall  x\;\forall  z \; \beta$. In addition, 
$( \neg \exists  y \; \alpha )^\forall$ and $( \neg \alpha )^\forall$ (using nnf) are both $\forall  x\;\forall  y \; \neg\alpha$. 
Our implementation avoids such repetitions. We provide some figures about the final amount of different falsity- and truth-tests in Section \ref{section:experimentation}. \\

The truth-tests for an axiom $\phi$ are the negations of all the falsity-tests for $\phi$. Hence, in the case of an axiom of the form $\phi[\alpha \vee \beta]$, we generate (among others) the truth-tests $( \neg \alpha )^\exists$, $( \neg \beta )^\exists$, $( \alpha )^\exists$ and $( \beta )^\exists$. Therefore, the function $\TT{\_}$ is simply defined on the basis of $\FT{\_}$ as follows.
\begin{definition} \label{defn:TruthTests}
For any af-nnf formula $\phi$, the function $\TT{\_}$ is defined as:
\begin{displaymath}
\TT{\phi} \; = \; \{ \neg \theta \; | \; \theta \in \FT{\phi} \}
\end{displaymath}
\end{definition}

\begin{example} \label{ex:sibling-TT}
For axiom (\ref{formula:oneSortedBroodSibling}) in Example \ref{ex:sibling-FT}, we obtain the following truth-tests
by negation of falsity-tests (\ref{goal:oneSortedBroodSiblingDisjunctionAlphaFalsityTest}-\ref{goal:oneSortedBroodSiblingDisjunctionNegBetaFalsityTest})
\begin{equation} \label{goal:oneSortedBroodSiblingDisjunctionAlphaTruthTest}
\begin{array}{rl}
\exists  m_1 \; \exists  m_2 \; \exists  b \; ( \hspace{-6pt} &  instance( m_1,Organism) \; \wedge \\
 &  instance( m_2,Organism) \; \wedge \\
 &  instance( b,Brood) \; \wedge \\
 & member( m_1, b) \; \wedge \; member( m_2, b) \; )
\end{array}
\end{equation} 
\begin{equation} \label{goal:oneSortedBroodSiblingDisjunctionBetaTruthTest}
\exists  m_1 \; \exists  m_2 \; ( \; \neg \; sibling( m_1, m_2) \; ) 
\end{equation} 
\begin{equation} \label{goal:oneSortedBroodSiblingDisjunctionNegAlphaTruthTest}
\begin{array}{rl}
\exists  m_1 \; \exists  m_2 \; \exists  b \;  ( \hspace{-6pt} & \neg \; instance( m_1,Organism) \; \vee \\
 & \neg \;  instance( m_2,Organism) \; \vee  \\
 & \neg \; instance( b,Brood) \; \vee  \\
 & \neg \;member( m_1, b) \; \vee  \; \neg \; member( m_2, b) \; )
\end{array}
\end{equation} 
\begin{equation} \label{goal:oneSortedBroodSiblingDisjunctionNegBetaTruthTest}
\exists  m_1 \; \exists  m_2 \; ( \; sibling( m_1, m_2) \; )
\end{equation}
Conjecture (\ref{goal:oneSortedBroodSiblingDisjunctionAlphaTruthTest}) is classified as unknown (as expected, because its corresponding falsity-test is classified as proved), while conjectures (\ref{goal:oneSortedBroodSiblingDisjunctionBetaTruthTest}-\ref{goal:oneSortedBroodSiblingDisjunctionNegBetaTruthTest}) are classified as proved, because ATPs find a proof. 
Since, from Example \ref{ex:sibling-FT}, we already know that axiom (\ref{formula:oneSortedBroodSibling}) is defective, this additional testing has no effect on the grade of suitability of axiom (\ref{formula:oneSortedBroodSibling}). 
However, in the hypothetical case that falsity-test (\ref{goal:oneSortedBroodSiblingDisjunctionAlphaFalsityTest})
would have been classified as unknown (instead of proved), then axiom (\ref{formula:oneSortedBroodSibling}) would be considered partially suitable, since three truth-tests are proved, but one is unknown. 
\end{example}

\section{A Detailed Example} \label{section:example}

In this section, we provide a complete example of the sets of tests that are obtained from a formula by using the definitions introduced in Section \ref{section:tests}.

The following three axioms are obtained from \SUMO{} by renaming (for brevity) the original predicate and function symbols as follows: the predicates \textPredicate{ instance} and \textPredicate{component} have been renamed as $i$ and $c$ respectively; the constants \textConstant{CorpuscularObject}, \textConstant{Atom}, \textConstant{Proton} and \textConstant{Electron} have been renamed as $CO$, $At$, $Pr$ and $El$ respectively; and
the variables \textVariable{ATOM}, \textVariable{PROTON} and \textVariable{ELECTRON} have been renamed as $ a$, $ p$ and $ e$ respectively.
\begin{equation}
\label{axiom:domainComponent1} 
domain(c,1,CO)
\end{equation}
\begin{equation}
\label{axiom:domainComponent2} 
domain(c,2,CO)
\end{equation}
\begin{equation}
\label{axiom:componentAtom}
\forall a\;
(\;
i(a,At) \;\rightarrow\;
\exists p\;\exists e\;
(
c(p,a)\wedge c(e,a) 
\wedge i(p,Pr) \wedge i(e,El)
)
\;)
\end{equation}
Using the type information in axioms (\ref{axiom:domainComponent1}-\ref{axiom:domainComponent2}) (see \cite{ALR12}), (\ref{axiom:componentAtom}) is transformed into a rule axiom in \ADIMENSUMO{} v2.4 that is logically equivalent to the following sentence in af-nnf:
\begin{equation} \label{formula:exampleOneSortedComponentAtom}
\begin{array}{lll}
\phi \; = \; 
\forall  a \; ( \hspace{-6pt} & \multicolumn{2}{l}{\overbrace{ ( \neg i( a,CO) ) \vee ( \neg i( a,At) )}^{\psi_1} \; \vee} \\
 & \exists  p \; \exists  e \; ( \hspace{-6pt} & i( p,CO) \wedge i( e,CO) \; \wedge \\
 & & c( p, a) \wedge c( e, a) \; \wedge \\ 
 & & \underbrace{i( p,Pr) \wedge i( e,El)}_{\psi_2} \; ) \; )
\end{array}
\end{equation}
 Then
\begin{eqnarray*}
\FT{\phi} & = & \FT{\forall  a \; ( \psi_1 \vee \exists  p \; \exists  e \; \psi_2 )} \\
 & \stackrel{\forall}{=} & \FT{\psi_1 \vee \exists  p \; \exists  e \; \psi_2} \\
 & \stackrel{\vee}{=} & \FTP{\psi_1 \vee \exists  p \; \exists  e \; \psi_2} \\
 & \stackrel{\vee}{=} & \{ \; ( \psi_1 )^\forall,\; (\exists  p \; \exists  e \; \psi_2 )^\forall, \;( \neg \psi_1 )^\forall,\;(\neg \exists  p \; \exists  e \; \psi_2 )^\forall  \; \} \cup \\
 & & \hspace{2pt} \FTP{\psi_1} \;\cup\; \FTP{\exists  p \; \exists  e \; \psi_2}
\end{eqnarray*}
where $( \psi_1 )^\forall$, $( \neg \psi_1 )^\forall$, $( \exists  p \; \exists  e \; \psi_2 )^\forall$ and $( \neg \exists  p \; \exists  e \; \psi_2 )^\forall$ respectively denote the conjectures:
\begin{equation} \label{goal:exampleAntecedentOneSortedComponentAtom}
\forall  a \; ( \; ( \neg i( a,CO) ) \vee ( \neg i( a,At) ) \; )
\end{equation}
\begin{equation} \label{goal:exampleReverseAntecedentOneSortedComponentAtom}
\begin{array}{rl}
\forall  a \; \exists  p \; \exists  e \; (& i( p,CO) \wedge i( e,CO) \wedge c( p, a) \wedge c( e, a) \wedge \\
 & i( p,Pr) \wedge i( e,El) \; )
\end{array}
\end{equation}
\begin{equation} \label{goal:exampleReverseConsequentOneSortedComponentAtom}
\forall  a \; ( \; i( a,CO) \wedge i( a,At) \; )
\end{equation}
\begin{equation} \label{goal:exampleConsequentOneSortedComponentAtom}
\begin{array}{rl}
\forall  a \; \neg \; \exists  p \; \exists  e \; (&i( p, CO) \wedge i( e,CO) \wedge c( p, a) \wedge c( e, a) \wedge \\
 & i( p,Pr) \wedge i( e,El) \; )
\end{array}
\end{equation}
Then, we proceed to obtain the falsity-test proposed for $\psi_1$:
\begin{eqnarray*}
\FTP{\psi_1} & = & \FTP{( \neg i( a,CO) ) \vee ( \neg i( a,At) )} \\
 & \stackrel{\vee}{=} & \{ \; ( \neg i( a,CO) )^\forall, \; ( \neg i( a,At) )^\forall,\; ( i( a,CO) )^\forall,\; ( i( a,At) )^\forall\} \cup \\
 & & \hspace{25pt} \cup \;\FTP{\neg i( a,CO)}  \cup \FTP{\neg i( a,At)} \\
 & \hspace{-4pt}  \stackrel{literal}{=} \hspace{-4pt} &\{ \; ( \neg i( a,CO) )^\forall, \; ( \neg i( a,At) )^\forall,\; ( i( a,CO) )^\forall,\; ( i( a,At) )^\forall\}
\end{eqnarray*}
Next, the falsity-tests for $\exists  p \; \exists  e \; \psi_2$ are calculated as follows:
\begin{eqnarray*}
\FTP{\exists  p \; \exists  e \; \psi_2} & = & \{ \; ( \exists  e \; \psi_2 )^\forall, ( \neg \exists  e \; \psi_2 )^\forall \; \} \cup \FTP{\exists  e \; \psi_2} \\
 & \stackrel{\exists}{=} & \{ \; ( \exists  e \; \psi_2 )^\forall, ( \neg \exists  e \; \psi_2 )^\forall \; \} \cup \\
 & & \hspace{25pt} \{ \; ( \psi_2 )^\forall, ( \neg ( \psi_2 ) )^\forall \; \} \cup \FTP{\psi_2}
\end{eqnarray*}
Of the above tests, $( \neg \exists  e \; \psi_2 )^\forall$ and $( \neg \psi_2 )^\forall$ are identical. Finally,  the tests that are obtained from $\psi_2$ are (see Remark \ref{remark: n-ary}):
\begin{eqnarray*}
\FTP{\psi_2} & = & \FTP{i( p,CO) \wedge i( e,CO) \wedge c( p, a) \wedge c( e, a) \wedge i( p,Pr) \wedge i( e,El)}\\
 & = &  \{ \; ( i( p,CO) )^\forall, \; ( i( e,CO) )^\forall,\; c( p, a) )^\forall,   \\
 & & \hspace{10pt}   (c( e, a))^\forall ,\; ( i( p,Pr) )^\forall,\; ( i( e,El) )^\forall \} \; \cup \\ 
 & & \{ \; ( \neg i( p,CO) )^\forall, ( \neg i( e,CO) )^\forall, \; \neg c( p, a) )^\forall, \\
 & & \hspace{10pt}  ( \neg c( e, a) )^\forall \;  , ( \neg i( p,Pr) )^\forall  , ( \neg i( e,El) )^\forall \; \} 
\end{eqnarray*}
Some of the tests in $\FTP{\psi_2}$ are identical up to  renaming of quantified variables. For example, the tests $( i( p,CO) )^\forall$ and $( i( e,CO) )^\forall$ are identical {\it up to renaming (u.t.r.)} of the universally quantified variable. Additionally, test $( i( a,CO) )^\forall$, which is also u.t.r. identical to both, is already in the set  $\FTP{\psi_1}$. Consequently, the number of tests in $\FT{\phi}$ can be substantially reduced by removing duplicates. In total, 28 falsity-tests are obtained from $\FT{\phi}$, from which 13 redundant tests can be removed. It is worth noting that additional duplicated tests can arise when combining sets of falsity-tests for different axioms. To sum up, the set of falsity-tests for $\phi$ is:
\begin{eqnarray*}
\FT{\phi} & = & \{ \; ( \psi_1 )^\forall, ( \neg \psi_1 )^\forall, ( \exists  e \; \psi_2 )^\forall, ( \neg \exists  e \; \psi_2 )^\forall, \\
 & & \hspace{10pt} ( \neg i( a,CO) )^\forall, ( i( a,CO) )^\forall, ( \neg i( a,At) )^\forall, \\
 & & \hspace{10pt} ( i( a,At) )^\forall, ( \psi_2 )^\forall, ( c( p, a) )^\forall, \\
 & & \hspace{10pt} ( \neg c( p, a) )^\forall, ( i( p,Pr) )^\forall, ( \neg i( p,Pr) )^\forall, \\
 & & \hspace{10pt} ( i( e,El) )^\forall, ( \neg i( e,El) )^\forall \; \}
\end{eqnarray*}

In Section \ref{section:defects} we report on defects found using this set of falsity-tests.

Regarding truth-tests, we obtain the following set of tests by negating the falsisty-tests in $\FT{\phi}$:
\begin{eqnarray*}
\TT{\phi} & = & \{ \; ( \neg \psi_1 )^\exists, ( \psi_1 )^\exists, ( \neg \exists  e \; \psi_2 )^\exists, ( \exists  e \; \psi_2 )^\exists, \\
 & & \hspace{10pt}  ( i( a,CO) )^\exists, ( \neg i( a,CO) )^\exists, ( i( a,At) )^\exists, \\
 & & \hspace{10pt} ( \neg i( a,At) )^\exists, ( \neg \psi_2 )^\exists, ( \neg c( p, a) )^\exists, \\
 & & \hspace{10pt} ( c( p, a) )^\exists, ( \neg i( p,Pr) )^\exists, ( i( p,Pr) )^\exists, \\
 & & \tab ( \neg i( e,El) )^\exists, ( i( e,El) )^\exists \; \}
\end{eqnarray*}

\section{Correctness} 
\label{section:correctness}

As mentioned in Section \ref{section:tests}, our falsity-tests are related to the search of redundancies in axioms. 
In this section, we introduce a precise notion of redundancy and prove that our falsity-tests are relevant (or correct) in the sense that they really detect redundancy. 

\begin{definition} \label{defn:redundancy1}
Let $\Phi$ be a set of sentences and $\phi$ a sentence such that $\phi \in \Phi$. We say that $\phi$ contains a {\em redundancy}  if there exists a sentence $\phi'[\alpha \vee \beta]$ that is logically equivalent to $\phi$ such that $\phi'[ ( \alpha \vee \beta ) / \gamma ]$ is $\Phi$-equivalent to $\phi$ for $\gamma \in \{ \alpha, \beta \}$. Moreover, when $\gamma = \alpha$ (resp. $\gamma = \beta$)
we say that $\beta$ (resp. $\alpha$) is redundant in the subformula $\alpha \vee \beta$.
\end{definition}
In words, redundancy means that some subformula (of a sentence) can be eliminated without loss of essential information, and indeed this subformula is redundant. The sentence $\phi'$ mentioned in the above definition is the one provided by the following Lemma \ref{lemma:subformula-FT}, i.e. the {\it formula associated to} $\phi$ by the pair of falsity-tests $\{ ( \delta )^\forall, \; ( \neg \delta )^\forall \}$ (see Remark \ref{re:formulaass}).

\begin{lemma} \label{lemma:subformula-FT}
For any af-nnf sentence $\phi$, if $\{ ( \delta )^\forall, ( \neg \delta )^\forall\} \subseteq \FT{\phi}$ then $\phi$ is logically equivalent to a formula of the form
\begin{displaymath}
\quantifier{ x} \; ( \; ( ( \quantifier{ y} \; \delta ) \vee \gamma ) \wedge \psi \; )
\end{displaymath}
for some formulas $\gamma$ and $\psi$ and some (possibly empty) prefixes of quantifiers $\quantifier{ x}$ and $\quantifier{ y}$. Moreover, $\gamma$ is always different from the constant $\false$, but $\psi$ could be the constant $\true$.
\end{lemma}
\begin{proof} 

The proof is by induction on the number of calls to $\FT{\_}$ and $\FTP{\_}$ that are made to check that $\{ ( \delta )^\forall,\; ( \neg \delta )^\forall \} \subseteq \FT{\phi}$.

We are going to prove that for any subformula $\chi$ (could be a non-sentence) of $\phi$, if $\{ ( \delta )^\forall, \; ( \neg \delta )^\forall \} \subseteq \FT{\chi}$ then $\chi$ is logically equivalent to some formula of the form
$
\quantifier{ x} \; ( \; ( ( \quantifier{ y} \; \delta ) \vee \gamma ) \wedge \psi \; )
$
and, moreover, variables $\overline{ x}, \overline{ y}$ are the variables in some quantifier of a subformula of $\chi$. It is worth noting that $\phi$ has no variable clashing
(see Section \ref{section:tests}), hence the variables occurring in some quantifier of a subformula $\chi$ cannot appear in any quantifier of some subformula of $\phi$ different from $\chi$. The base step is when $\chi = \delta \vee \gamma$ so that
$\FT{\phi} = \FTP{\delta \vee \gamma} = \{ ( \delta )^\forall, \; ( \gamma )^\forall, \;( \neg \delta )^\forall, \; ( \neg \gamma )^\forall \} \cup \FTP{\delta} \cup \FTP{\gamma}$. So that the property holds for empty prefixes of quantifiers and $\psi = \true$.
Symmetrically for $\chi =  \gamma \vee \delta$.

For the inductive step, we distinguish the following cases according to the definition of $\FT{\_}$ and $\FTP{\_}$:
\begin{itemize}
\item $\FT{\chi} = \FTP{  \alpha_1 \wedge \beta_1 ) \vee \beta} \supset 
\FTP{\alpha_1 \wedge \beta_1} \cup \FTP{\beta}$,
and then
$\{ ( \delta )^\forall,\; ( \neg \delta )^\forall \} \supseteq \FTP{\alpha_1}$
or
$\{ ( \delta )^\forall, \; ( \neg \delta )^\forall \} \subseteq \FTP{\beta_1}$.

We prove only the case $\{ ( \delta )^\forall, \; ( \neg \delta )^\forall \} \supseteq \FTP{\alpha_1}$, since for $\beta_1$ the proof is identical by commutativity of conjunction.

By induction hypothesis, if $\{ ( \delta )^\forall, \; ( \neg \delta )^\forall \} \subseteq \FTP{\alpha_1}$, then there exists $\gamma'$ and $\psi'$ such that:
$$
\alpha_1 \; \equiv \; \quantifier{ x} \; ( \; ( ( \quantifier{ y} \; \delta ) \vee \gamma' ) \wedge \psi' \; )
$$
where variables $\overline{ x}, \overline{ y}$ cannot appear either in $\beta_1$ or in $\beta$. Therefore:
\begin{eqnarray*}
\chi & \equiv & ( \; ( \; \quantifier{ x} \; ( \; ( ( \quantifier{ y} \; \delta ) \vee \gamma' ) \wedge \psi' \; ) \; ) \wedge \beta_1 \; ) \vee \beta \\
 & \equiv & ( \; \quantifier{ x} \; ( \; ( ( \quantifier{ y} \; \delta ) \vee \gamma' ) \wedge ( \psi' \wedge \beta_1 ) \; ) \; ) \vee \beta \\
 & \equiv & \quantifier{ x} \; ( \; (( \quantifier{ y} \; \delta ) \vee ( \gamma' \vee \beta ) ) \wedge ( ( \psi' \wedge \beta_1 ) \vee \beta ) \; )
\end{eqnarray*}
Hence, we can take $\gamma' \vee \beta$ as $\gamma$ and 
$( \psi' \wedge \beta_1 ) \vee \beta$ as $\psi$, and the property of variables $\overline{ x}, \overline{ y}$ is trivially preserved.

\item $\FT{\chi} = \FTP{( \forall z \; \alpha ) \vee \beta} \supset 
\FTP{\alpha} \cup \FTP{\beta}$,
and then
$\{ ( \delta )^\forall, \; ( \neg \delta )^\forall \} \subseteq \FTP{\alpha}$.
By induction hypothesis, there exist $\gamma'$ and $\psi'$ such that:
$$
\alpha \; \equiv \; \quantifier{ x} \; ( \; ( ( \quantifier{y} \; \delta ) \vee \gamma' ) \wedge \psi' \;)
$$
Therefore:
$
\chi \; \equiv \; \forall  z \; ( \; \quantifier{ x} \; ( \; ( ( \quantifier{ y} \; \delta ) \vee \gamma' ) \wedge \psi' \; ) \; ) \vee \beta
$.
Since $ z$ and $\overline{ x}, \overline{ y}$ do not appear in $\beta$, we have that:
$$
\chi \; \equiv \; \forall  z \; \quantifier{ x} \; ( \; ( ( \quantifier{ y} \; \delta ) \vee ( \gamma' \vee \beta ) ) \wedge ( \psi' \vee \beta ) \; )
$$
Thus, the property holds for an extension of the outermost prefix with $ z$, and for $\gamma' \vee \beta$ as $\gamma$ and $\psi' \vee \beta$ as $\psi$. 

\item The proof for $\FT{\chi} = \FTP{ ( \exists  z \; \alpha ) \vee \beta } \supset \FTP{\alpha} \cup \FTP{\beta}$ and $\{ ( \delta )^\forall,\; ( \neg \delta )^\forall \} \subseteq \FTP{\alpha}$, is identical to the previous one.

\item Suppose that $\FT{\chi} = \FT{\alpha \wedge \beta} =
\FT{\alpha} \cup \FT{\beta}$ and $\{ ( \delta )^\forall, \; ( \neg \delta )^\forall \} \subseteq \FT{\alpha}$.
By induction hypothesis
$$
\alpha \; \equiv \; \quantifier{ x} \; ( \; ( ( \quantifier{ y} \; \delta ) \vee \gamma' ) \wedge \psi' \; )
$$
for some $\gamma'$ and $\psi'$. Therefore, since $\overline{ x}$ does not appear in $\beta$
\begin{eqnarray*}
\chi & \equiv & ( \; \quantifier{ x} \; ( \; ( ( \quantifier{ y} \; \delta ) \vee \gamma' ) \wedge \psi' \; ) \; ) \wedge \beta \\
 & \equiv & \quantifier{ x} \; ( \; ( ( \quantifier{ y} \; \delta ) \vee \gamma' ) \wedge ( \psi' \wedge \beta ) \; )
\end{eqnarray*}
and the lemma property is true for $\gamma = \gamma'$ and $\psi = \psi' \wedge \beta$.

\item For $\FT{\chi} = \FT{ \forall  z \; \alpha } \supset \FT{\alpha}$ and $\{ ( \delta )^\forall, \; ( \neg \delta )^\forall \} \subseteq \FT{\alpha}$.
By induction hypothesis
$$
\alpha \; \equiv \; \quantifier{ x} \; ( \; ( ( \quantifier{ y} \; \delta ) \vee \gamma ) \wedge \psi \; )
$$
for some $\gamma$ and $\psi$. Therefore:
$$
\chi \; \equiv \; \forall  z \; \quantifier{ x} \; ( \; ( ( \quantifier{ y} \; \delta ) \vee \gamma ) \wedge \psi \; )
$$
This ensures the lemma property by enlarging the outermost prefix of quantifiers with $\forall  z$.

\item The proof for $\FT{\chi} = \FT{ \exists  z \; \alpha } \supset \FT{\alpha}$ and  $\{ ( \delta )^\forall, \; ( \neg \delta )^\forall \}\subseteq \FT{\alpha}$ is identical to the previous one.
\end{itemize}
Therefore, for any $\{ ( \delta )^\forall,\; ( \neg \delta )^\forall \} \subseteq \FT{\phi}$, the axiom $\phi$ is logically equivalent to a formula of the form $\quantifier{ x} \; ( \; ( ( \quantifier{ y} \; \delta ) \vee \gamma ) \wedge \psi \; )$.
\end{proof}

\begin{remark}
\label{re:formulaass}
In what follows, we say that $( \quantifier{ y} \; \delta )\vee \gamma )$ is the {\it (sub)formula associated to} $\phi$ by the pair of falsity-tests $\{ ( \delta )^\forall, \; ( \neg \delta )^\forall \}$.
\end{remark}

Next, we introduce the notion of {\it relevant} falsity-test in order to set out the correctness result in Theorem \ref{thm:FT-correctness}.

\begin{definition} \label{defn:correctness-FT}
Let $\Phi$ be any set of af-nnf sentences and $\phi$ any sentence such that $\phi \in \Phi$. If $( \delta )^\forall \in \FT{\phi}$ (resp. $( \neg \delta )^\forall \in \FT{\phi}$), we say that $( \delta )^\forall$ (resp. $( \neg \delta )^\forall$) is a {\it relevant falsity-test for $\Phi$} whenever $\Phi \models \delta^\forall$ (resp. $\Phi \models ( \neg \delta )^\forall$).
\end{definition}
If $\delta^\forall$ or $( \neg \delta )^\forall$ is relevant, then the redundant subformula of $\phi'$ is a superformula of $\delta$ and $\phi$ contains some redundancy. Redundant subformulas reveal the existence of defects. In addition, the proof obtained for falsity-tests can assist the correction of defects, but the correction itself is still a manual task. 
Some real examples on this issue are described in Section \ref{section:defects}.

Next, we provide a formal proof of the relevance of falsity-tests.
\begin{theorem} \label{thm:FT-correctness}
For any consistent set of af-nnf sentences $\Phi$ such that $\phi \in \Phi$, each conjecture in $\FT{\phi}$ is a relevant falsity-test for $\Phi$.
\end{theorem}
\begin{proof}
Let $\{ ( \delta )^\forall, \; ( \neg \delta )^\forall \} \subseteq \FT{\phi}$. By Lemma \ref{lemma:subformula-FT}, $\phi$ is logically equivalent to:
\begin{displaymath}
\phi' \; = \; \quantifier{ x} \; ( \; ( \quantifier{ y} \; \delta ) \vee \gamma \; ) \wedge \psi \; )
\end{displaymath}
Hence, we check the following two statements:
\begin{itemize}
\item[\rm (a)] If $\Phi \models ( \delta )^\forall$,
then $\Phi \models \phi' \leftrightarrow \phi'[ ( \; ( \quantifier{ y} \; \delta ) \vee \gamma \; ) / ( \quantifier{ y} \; \delta ) ]$.
\item[\rm (b)] If $\Phi \models (\neg\delta)^\forall$, then $\Phi \models \phi' \leftrightarrow \phi'[ ( \; ( \quantifier{ y} \; \delta ) \vee \gamma \; ) / \gamma ]$.
\end{itemize} 
Note that each statement not only ensures that the corresponding subformula is redundant in $\phi'$, but also the substitution specifies what the redundant subformula is. 

In order to check (a) and (b), we proceed by substitutivity of subformulas that are  $\Phi$-equivalent.
For statement (a), if $\Phi \models ( \delta )^\forall$, then it is trivial that $\Phi \models ( \quantifier{ y} \; \delta )^\forall$. Hence, $( \quantifier{ y} \; \delta ) \vee \gamma$ is $\Phi$-equivalent to $\quantifier{ y} \; \delta$.
For statement (b), if $\Phi \models ( \neg \delta)^\forall$, then 
$\Phi \models ( \neg \overline{Qy} \; \delta )^\forall$. Hence, $( \quantifier{ y} \; \delta) \vee \gamma$ is $\Phi$-equivalent to $\gamma$.
\end{proof}

Theorem \ref{thm:FT-correctness} ensures the correctness of our method. By Theorem \ref{thm:FT-correctness}, whenever a falsity-test $\alpha\in FT(\phi)$ is proved to be entailed by the ontology (where $\phi$ belongs to), we can ensure that there is a {\em relevant redundancy} in the axiom $\phi$. Since truth-tests are the negations of falsity-tests, the fact of proving a truth-test guarantees the absence of a possible redundancy. Consequently, the complete suitability of an axiom $\phi$ in the ontology depends on proving every test in $\TT{\phi}$.\\

According to the above definitions and results, our method relies on testing whether the whole ontology entails a set of tests. In particular, let $\alpha$ be a test that has been generated from an axiom $\phi$, then $\phi$ is included in the premises used by the ATP to check whether $\alpha$ is entailed. From the practical point of view this is the easiest way to perform the testing, since the set of premises is fixed for testing the whole set of generated tests. From the theoretical point of view, Theorem \ref{thm:FT-correctness} ensures that whenever $\alpha$ is proved, there is a relevant redundancy in the axiom $\phi$.\footnote{A different issue, that we discuss in Section \ref{section:defects}, is to look for the defect that causes this redundancy. }
Moreover, we can ensure that deleting the axiom $\phi$ from the premises (when checking $\alpha$) would prevent some useful inferences revealing redundancies. Let us explain here a simplification of a real example to illustrate this matter.
\begin{example}
Consider an ontology $\Phi$ formed by three axioms:
\begin{eqnarray}
\label{r(x)r(y)}
     \phi_1 & = & \forall x \;\forall y\;(\;(p(x) \wedge q(x,y)) \rightarrow r(x) \;)  \\
\nonumber
    \phi_2 & = &  \forall x \;(\;p(x) \rightarrow \neg r(x) \;) \\
\nonumber
    \phi_3 & = &  \forall x \;(\;(\exists y\;q(x,y)) \rightarrow \neg r(x) \;) 
\end{eqnarray}
It is easy to see that a falsity-text in $\FT{\phi_1}$ is 
$\forall x \;\forall y\;(\;\neg p(x) \vee\neg q(x,y)\;)$ and also that this falsity-test is entailed by $\Phi$, but it is not entailed by $\Phi\setminus\{\phi_1\}$.
Hence, the redundancy of the antecedent of (\ref{r(x)r(y)})
can be detected only if 
$\phi_1$ belongs to the the set of premises. 
\end{example}

The above example is an abstraction and simplification of the first example explained in Subsection \ref{subsection:typos}, there we also explain that it is really caused by a simple typo in $\phi_1$. In fact, the typo in the real example corresponds to have written $r(x)$ instead of  $r(y)$ in $\phi_1$ above.

\section{Experimentation: Examples of Detected Defects} 
\label{section:defects}

In this section, we illustrate with examples the defects that we have detected ---using an implementation of our methodology--- in the ontologies \DOLCE{}, \KEPLER{} and \ADIMENSUMO{} (see Section \ref{section:FOL-Ontologies} for an introduction to them).
Our method is based on detecting redundancies. To be more precise our technique is focussed in looking for {\it redundant subformulas inside axioms} (using falsity-tests)  or their absence (using truth-tests).
Indeed, according to Theorem \ref{thm:FT-correctness}, whenever an ATP proves the entailment of a falsity-test $\alpha$ extracted from an axiom $\phi$, there is a disjunction $\alpha\vee\beta$ that is a subformula of $\phi$ such that either $\alpha$ or $\beta$ are redundant in the sense that the disjunction $\alpha\vee\beta$ can be substituted by one of its disjuncts ($\alpha$ or $\beta$) preserving equivalence. So, in some sense, these redundancies are the defects we are basically looking for, but there are different causes that provoke a subformula to be technically redundant in the sense of Theorem \ref{thm:FT-correctness}. On one hand, some redundancies arise because there is a different problem either in the concerned axiom or in some set of axioms that define related concepts. For example, a typo in the concerned axiom or some inaccuracy in the axioms defining a concept related with the concerned axiom.
In other words, sometimes the defect can be fixed as expected (i.e. by eliminating a redundant subformula), but sometimes the detected redundancy could be caused by a different defect. 
For example, the redundancy detected in (\ref{r(x)r(y)}) would be (indeed, it is) really caused by a typo: the consequent of (\ref{r(x)r(y)}) should by $r(y)$ instead of $r(x)$.
On the other hand, sometimes redundant axioms are intentionally introduced with different objectives, e.g.  a) to facilitate the proof of conjectures; b) to maintain an uniform style of modelling; c) to communicate design decisions; d) to improve the documentation of the ontology. Additionally, automatic transformations often introduce redundant subformulas. Since the correction of defects is not automatic but manual, the ontologist/expert has to decide whether the detected redundancies need to be corrected or not. 
In Subsections \ref{subsection:typos}-\ref{subsec:incorrect-axioms}, we provide some examples of detected redundancies that --for that reason-- do not require any correction, along with other examples that require correction.
For the sake of presentation we split the founded defects in four categories: {\it typos}, {\it redundant axioms}, {\it redundant subformulas (in axioms)} and {\it incorrect (inaccurate) axioms}. We provide examples of each category in the following four subsections.

\subsection{Typos} \label{subsection:typos}

{\it Typos} are syntactical errors that are very simple to correct. A first example of typo was in the following  axiom of \ADIMENSUMO{}:
\begin{eqnarray}
\nonumber
\forall c\;\forall g\;(\;
& \hspace{-5mm} (  & \hspace{-5mm} 
instance(c,OrchestralConducting)\wedge patient(c,g) )\\
& \rightarrow &
instance(c,Orchestra)
\;\;)
\label{axiom:WrongOrchestralConducting}
\end{eqnarray}
which aims to state that every participant in an instance of the process \textConstant{OrchestralConducting} is an instance of \textConstant{Orchestra}. However, the following falsity-test was proved:
\begin{equation}
\forall c\;\forall g\;(\;
 (  \;
\neg instance(c,OrchestralConducting)\\
\vee 
\neg patient(c,g) 
\;)
\label{test:WrongOrchestralConducting}
\end{equation}
%
This suggest that the antecedent of axiom (\ref{axiom:WrongOrchestralConducting}) is redundant, however the proof given by the ATP utilizes an axiom stating that the first argument of \textPredicate{patient} (i.e. $c$ in (\ref{axiom:WrongOrchestralConducting})) should be a \textConstant{Process}. By disjointness of the class \textConstant{Process} with the class \textConstant{Orchestra} --which is entailed passing through several superclasses of the latter-- it is deduced that  $c$ cannot be an instance of \textConstant{Orchestra}, hence using axiom (\ref{axiom:WrongOrchestralConducting}) itself, falsity-test (\ref{test:WrongOrchestralConducting}) is entailed. The hint that the conflict comes from the type of the first argument of \textPredicate{patient} made us realize that it is the second (but not the first) argument of \textPredicate{patient} that should be an \textConstant{Orchestra}. 
Hence, in axiom (\ref{axiom:WrongOrchestralConducting}), the third occurrence of $c$ is a typo. In \ADIMENSUMO{} v2.6, it has been replaced by $g$:
\begin{eqnarray}
\nonumber
\forall c\;\forall g\;(\;
& \hspace{-5mm} (  & \hspace{-5mm} 
instance(c,OrchestralConducting)\wedge patient(c,g) )\\
& \rightarrow &
instance(g,Orchestra)
\;\;)
\label{axiom:CorrectOrchestralConducting}
\end{eqnarray}

An example of a typo that was detected in \KIFDOLCE{} is the following axiom, where \textPredicate{sb} stands for the subsumption relation and \textPredicate{psb} stands for the proper subsumption relation:
\begin{eqnarray}
\nonumber
\forall w\;\forall f\;\forall g\; 
(\; 
&  \hspace{-5mm} (  &
 \hspace{-5mm} 
world(w)\wedge universal(f)\wedge universal(g)\;)\\
& \rightarrow  &
(\;psb(w,f,g)\leftrightarrow(sb(w,f,g)\wedge\neg sb(w,f,g))\;)
\;)
\label{axiom:wrongProperlySubsuming}
\end{eqnarray}
The following falsity-text was generated from this axiom:
\begin{equation}
\label{goal:properlySubsuming}
\forall w\;\forall x\;\forall y\;
(\;
\neg psb(w,x,y)
\;)
\end{equation}
It was classified as proved, and its proof reveals that, in axiom (\ref{axiom:wrongProperlySubsuming}), the relation \textPredicate{psb} was incorrectly defined:
the last two occurrences of $f$ and $g$ in axiom (\ref{axiom:wrongProperlySubsuming}) 
were swapped, i.e. the last literal in (\ref{axiom:wrongProperlySubsuming}) should be
$
\neg sb(w,g,f).
$
%
After correcting this typo, 606 falsity-tests turned from proved into unknown.

\subsection{Redundant Axioms}

{\it Redundant axioms} are axioms that do not add any reasoning power to the ontology, because they are already entailed. An example is the \ADIMENSUMO{} axiom which yields the formula (\ref{formula:CenterOfCircle}) in Section \ref{section:methodology}.
As explained there, this defect was detected by means of falsity-test (\ref{goal:CenterOfCircle}). In \ADIMENSUMO{} v2.6 axiom (\ref{formula:CenterOfCircle}) is removed, though a proper axiomatization of \textConstant{CenterOfCircleFn} would be more convenient.\\
We have also found an example of a redundant axiom in \KIFDOLCE{} (after we had corrected axiom (\ref{axiom:wrongProperlySubsuming})).
The role of the following axiom is to define the disjointness relation \textPredicate{dj} of universals:
%
\begin{eqnarray}
\nonumber
\forall w\;\forall f\;\forall g\; 
(\; 
&  \hspace{-5mm} (  &
 \hspace{-5mm} 
world(w)\wedge universal(f)\wedge universal(g)\;)\\
& \rightarrow  &
(\;dj(w,f,g)\leftrightarrow\phi\;)
\;)
\label{axiom:disjointUniversals}
\end{eqnarray}
where $\phi$ is the formula that states  the disjointness of universals \textVariable{F} and \textVariable{G}, but it is not relevant for the present discussion.
The following falsity-test is proved:
\begin{equation}
\label{goal:world} 
\forall w\;(\;\neg world(w)\;)
\end{equation}
Therefore, the negation of the antecedent of axiom (\ref{axiom:disjointUniversals}) is inferred from the ontology.
Hence, the consequent of axiom (\ref{axiom:disjointUniversals}) is redundant, and thus the axiom itself is classified as redundant.
This means that \textPredicate{dj} has no associated definition (axiomatization) in \KIFDOLCE{}. By analysing the proof provided by the ATP, we discover that the axiomatization in \KIFDOLCE{} prevents the introduction of worlds and particulars. In fact, in \KIFDOLCE{} the definition of any world or any particular yields to a contradiction. This defect does not seem to be trivial to fix and it is beyond the scope of this paper. In \KIFDOLCE{}, there are 91 different falsity-tests that were proved due to this defect.\\

In the ontology \KEPLER{}, the following atom:
\begin{equation} \label{formula:truthValue}
p(S(Bool,T))
\end{equation}
encodes the boolean constant for truth as a formula, indeed it is  called {\it aTRUTH}.
Atom (\ref{formula:truthValue}) is used as a subformula in 19 axioms, which produces 14 unique falsity-tests enabling the detection of the redundant uses of $p(S(Bool,T)$. For example, $p(S(Bool,T))$ is redundantly used in the following axiom (called {\it aREFLu\_CLAUSE}):
\begin{equation} \label{formula:reflexivity}
\forall  a \; \forall  x \; ( \; ( S( a, x) = S( a, x) ) \; \to \; p(S(Bool,T) \; )
\end{equation}
Indeed, for this axiom, we generate the  falsity-test:
\begin{equation} \label{goal:reflexivity}
\forall  a \; \forall  x \; ( \; S( a, x) = S( a, x) \; )
\end{equation}
which is easily proved since equality is defined as being reflexive in FOL. Consequently, axiom (\ref{formula:reflexivity}) is redundant.

\subsection{Redundant Subformulas}
\label{subsec:redsubf}
{\it  Redundant subformulas (in axioms)} can be simply removed from axioms without affecting the set of conjectures that can be entailed from the ontology. This kind of redundancy is often introduced by automatic transformation of axioms, such as the translation of {\it domain axioms} in \ADIMENSUMO{} or the transformation of \KIFDOLCE{} into \CASLDOLCE{}. An example is given by the following axioms in \CASLDOLCE{}, which define the disjointness of the top classes {\it endurant}, {\it perdurant} and {\it quality} with {\it abstract}:
\begin{eqnarray}
 & \forall  y_0 \; ( \; aB( y_0) \; \to \; pT(y_0) \;) & \label{formula:AbstractCASLDOLCE} \\
 & \forall  x_1 \; ( \; eDorPDorQ(x_1) \; \to \; pT(x_1)) \;) & \label{formula:EndurantPerdurantQualityCASLDOLCE} \\
 & \forall  x_0 \; ( \; pT( x_0) \; \to \; \neg ( aB(x_0) \wedge eDorPDorQ(x_0) ) \; ) & \label{formula:partitionCASLDOLCETriviallyRedundant}
\end{eqnarray}
It is easy to see that the falsity-test 
\begin{equation} \label{formula:partitionCASLDOLCE}
\forall  x_0 \; \neg ( aB(x_0) \wedge eDorPDorQ(x_0) )
\end{equation}
is entailed by axioms (\ref{formula:AbstractCASLDOLCE}-\ref{formula:partitionCASLDOLCETriviallyRedundant}).
This ensures that the antecedent of axiom (\ref{formula:partitionCASLDOLCETriviallyRedundant}) is redundant. Therefore, we have replaced axiom  (\ref{formula:partitionCASLDOLCETriviallyRedundant})
with (\ref{formula:partitionCASLDOLCE}).\\

\subsection{Incorrect Axioms}
\label{subsec:incorrect-axioms}

{\it Incorrect (or inaccurate) axioms} are sentences giving an inaccurate definition of the term they aim to define.\\
A first example of an incorrect axiom is given by axiom (\ref{axiom:BroodSibling}) from \ADIMENSUMO.
%
Falsity-test  (\ref{goal:oneSortedBroodSiblingDisjunctionAlphaFalsityTest})
is proved, suggesting that the antecedent of axiom (\ref{axiom:BroodSibling}) is redundant. The proof provided by the ATP reveals that the problem is related with
axiom (\ref{axiom:siblingIrreflexiveRelation}) which ensures that the relation \textPredicate{sibling} is irreflexive.
Consequently, in \ADIMENSUMO{} v2.6, we correct this defect by replacing axiom (\ref{axiom:BroodSibling}) with:
%
\begin{eqnarray}
\nonumber
\forall m_1\;\forall m_2\; \forall b\; 
(\; 
\hspace{-5mm} & ( & \hspace{-5mm} instance(b,Brood) \wedge 
member(m_1,b) \wedge
member(m_2,b)\wedge m_1 \neq m_2
) \\
& \;\rightarrow \; &
sibling(m_1,m_2)       \label{axiom:BroodSiblingcorrected}
\;)
\end{eqnarray}

Another example of an incorrect axiom is given in Section \ref{section:example}: falsity-test (\ref{goal:exampleAntecedentOneSortedComponentAtom}) is proved, and therefore the antecedent of the implication in axiom (\ref{axiom:componentAtom}) seems to be redundant. However, the proof of this falsity-test reveals that the class \textConstant{Atom} (abbreviated $At$) is a subclass of \textConstant{Substance} and the latter does not have common instances with \textConstant{CorpuscularObject} (abbr. $CO$), because \textConstant{Substance} and $CO$ are disjoint classes. By axiom (\ref{axiom:domainComponent2}), on the domain of predicate $c$ (recall \textPredicate{component}), 
the variable $a$ 
must be an instance of $CO$. We realized that the redundancy of the antecedent of axiom (\ref{axiom:componentAtom})
is due to this misuse of the relation $c$ (recall $component$) in the consequent of (\ref{axiom:componentAtom}). In \ADIMENSUMO{} v2.6, we fix this defect by replacing relation \textPredicate{component} (abbr. $c$) by  \textPredicate{part} in axiom (\ref{axiom:componentAtom}). The new axiom is:
\begin{equation}
\label{axiom:repaired}
\forall a\;
(\;
i(a,At) \;\rightarrow\;
\exists p\;\exists e\;
(
part(p,a)\wedge part(e,a) 
\wedge i(p,Pr) \wedge i(e,El)
)
\;)
\end{equation}
Before fixing that, falsity-test (\ref{goal:exampleConsequentOneSortedComponentAtom}) was also proved as \textConstant{Proton} (abbr. $Pr$) and \textConstant{Electron} (abbr. $El$) are also subclasses of \textConstant{Substance} and, therefore, they do not have common instances with $CO$ (\textConstant{CorpuscularObject}). This proved falsity-test is related to the redundancy of the consequent of axiom (\ref{axiom:componentAtom}), however the proof reveals the same misuse of the relation $c$ (or $component$), so that the repaired axiom (\ref{axiom:repaired}) also serves to fix this problem.

\section{Experimentation Results for \DOLCE{}, \KEPLER{} and \ADIMENSUMO{}} 
\label{section:experimentation}

In this section, we report on the experimentation results we obtained by testing three FOL ontologies: \DOLCE{}, \KEPLER{} and \ADIMENSUMO{}. 


\begin{table}[h!] \centering
\caption{\label{table:Testing} Number of (tested) rules and (unique) falsity-tests for each ontology
}
\begin{tabular}{lrrrr}
\hline 
 & \multicolumn{2}{c}{{Rules}} & \multicolumn{2}{c}{{Falsity-Tests}} \\
 & \multicolumn{1}{c}{{Total}} & \multicolumn{1}{c}{{Tested}} & \multicolumn{1}{c}{{Total}} & \multicolumn{1}{c}{{Unique}} \\
\hline \\[-10pt]
{\KIFDOLCE{}} & 257 & 215 & 10,151 & 2,138 \\
{\CASLDOLCE{}} & 416 & 378 & 6,637 & 1,266 \\
{\KEPLER{}} & 78,225 & 5,753 & 163,751 & 37,698 \\
{\ADIMENSUMO{} v2.4} & 2,785 & 1,622 & 27,392 & 7,996 \\
{\ADIMENSUMO{} v2.6} & 2,799 & 1,629 & 27,499 & 8,010 \\
\hline
\end{tabular}
\end{table}

We consider two different versions of \DOLCE{} and \ADIMENSUMO{}. In Table \ref{table:Testing}, we provide some general figures about the three ontologies.
Given an axiom $\phi$, our method generates at least one test for $\phi$ whenever the af-nnf of $\phi$ contains at least one disjunction connective. Hence, axioms that do not satisfy this condition --e.g. universally closed conjunctions of literals-- are not tested by our method.
In Table \ref{table:Testing}, the two columns ``Rules" respectively stand for
the total number of rule axioms in each ontology (also reported in Table \ref{table:FOLOntologiesFigures}) and the number of tested rules. 
In the columns ``Falsity-tests" we provide the total number of falsity-tests generated for each ontology and the number of different falsity-tests.
The number of truth-tests is equal to the number of falsity-tests by definition.

In the rest of this section, we report on the experimentation results for each ontology in the respective Subsections \ref{subsection:DOLCE}, \ref{subsection:KEPLER} and \ref{subsection:AdimenSUMO}.
Additionally, in Subsection \ref{subsection:AdimenSUMO} we also report on the improvement \ADIMENSUMO{} from v2.4 to v2.6.
In our experimentation, we have used three different vesions of the theorem prover Vampire \cite{RiV02,KoV13}, especificaly  v2.6, v3.0 and v4.1, since they produce different results.
We set an execution time limit of 600 seconds, running on an Intel\textregistered~Xeon
\textregistered~CPU E5-2640v3@2.60GHz with 2GB of RAM memory per processor. The five tested ontologies, the sets of tests and the program for its generation, and the execution reports are freely available at \url{http://adimen.si.ehu.es}.

\begin{threeparttable}[h!] \centering
\caption{ Experimentation results for \DOLCE{}}

\begin{subtable}{0.282\linewidth} 
\begin{tabular} {|l|cc|cccc|cc|cc|}
\hline \\[-10pt]
 & \multicolumn{2}{c|}{Tests} & \multicolumn{4}{c|}{{Axioms}} & \multicolumn{2}{c|}{ Coverage} & \multicolumn{2}{c|}{Time}\\
& \multicolumn{1}{c}{\#Ts.} & \multicolumn{1}{c|}{Pr.} &  \multicolumn{1}{c}{{D.}} & \multicolumn{1}{c}{{C.S.}} & \multicolumn{1}{c}{{P.S.}} & \multicolumn{1}{c|}{{U.}} & \multicolumn{1}{c}{{\#Ax.}} & \multicolumn{1}{c|}{{Perc.}} & 
\multicolumn{1}{c}{$\leq$1} & \multicolumn{1}{c|}{$>$120}  \\
\hline \\[-10pt]
{ FT} & 2,138 & 112  & 213 & - & - & - & 71 & 27.63\% & 20 & 89  \\
{ TT} & 2,138  & 407 & - & 0 & 2 & 0 & 44 & 17.12\% & 406 & 1\\
{ Total}  & 4276 & 519 & 213 & 0 & 2 & 0 & 77 & 29.96\% & 426 & 90 \\
\hline
\end{tabular}
\caption{ \label{subtable:TestingKIFDolce} \KIFDOLCE{} }
\end{subtable}

\vspace{-3mm}
\begin{subtable}{0.282\linewidth}
\begin{tabular} {|l|cc|cccc|cc|cc|}
\hline \\[-10pt]
 & \multicolumn{2}{c|}{Tests} & \multicolumn{4}{c|}{{Axioms}} & \multicolumn{2}{c|}{ Coverage} & \multicolumn{2}{c|}{Time}\\
& \multicolumn{1}{c}{\#Ts.} & \multicolumn{1}{c|}{Pr.} &  \multicolumn{1}{c}{{D.}} & \multicolumn{1}{c}{{C.S.}} & \multicolumn{1}{c}{{P.S.}} & \multicolumn{1}{c|}{{U.}} & \multicolumn{1}{c}{{\#Ax.}} & \multicolumn{1}{c|}{{Perc.}} & 
\multicolumn{1}{c}{$\leq$1} & \multicolumn{1}{c|}{$>$120}  \\
\hline \\[-10pt]
{ FT} & 1,266 & 30  & 30 & - & - & - & 73 & 17.55\% & 30 & 0  \\
{ TT} & 1,266  & 1,093 & - & 242 & 106 & 0 & 324 & 77.88\% & 1,075 & 0\\
{ Total}  & 2,532 & 1,123 & 30 & 242 & 106 & 0 & 325 & 78.13\% & 1,105 & 0 \\
\hline
\end{tabular}
\caption{ \label{subtable:TestingCASLDolce} \CASLDOLCE{}}
\end{subtable}

\vspace{-3mm}
\begin{tablenotes}
\small
\item Tests.- \#Ts.: Number of different tests;  Pr.: Number of Proved tests.
\item Axioms.-
        Number of axioms classified as D: defective, C.S.: Completely Suitable; 
\item \hspace{1.5cm} P.S.: Partially Suitable; U.: Unknown.
\item Coverage.- $\#$Ax.: Number of axioms that are involved in the set of proofs; 
\item \hspace{1.5cm} Perc.: Percentage over the total number of axioms in the ontology. 
\item Time.- $\leq s$: Number of proofs that takes less than $s$ seconds for $s\in\{1,10\}$ , 
\item \hspace{1cm} $>$120: Number of proofs that takes more than 120 seconds.
\end{tablenotes}
\label{table:Experimentation1}
\end{threeparttable}


\subsection{Testing \DOLCE{}} \label{subsection:DOLCE}

We consider two different FOL versions of \DOLCE{}: \KIFDOLCE{}, which was obtained by following the translation described in \cite{ALR12}, and \CASLDOLCE{}, which was obtained from the simplified translation of \DOLCE{} into CASL \cite{ABK02} that is available in Hets \cite{MML07}.

In Table \ref{subtable:TestingKIFDolce}, we provide the results we have obtained testing \KIFDOLCE{} with 2,138 different tests of each type (FT and TT).
The notes below the table explain the quantities and abbreviations. We would like to point out that among the 2,138 falsity-tests, only 112 (5.24\%) were proved, while 407 truth-tests were proved. 
The proved falsity-tests enable, at first sight, only 2 axioms (up to 215) be classified as partially suitable,
whereas the remaining 213 axioms (99.07\%) are classified as defective.
However, according to the coverage measure, only 77 axioms (29.96\% of total) are used by ATPs, of which 71 axioms are used in the proofs of falsity-tests. This low coverage is due to the fact that the addition of any constant representing a $world$ or a $particular$ yields to a contradiction in \KIFDOLCE{}, where $world$ and $particular$ are fundamental concepts. The small inconsistent susbset of formulas (caused by such addition) is repeatedly used by ATPs to prove the entailment of the falsity-tests. Hence, this is a defect of \KIFDOLCE{}, though it does not mean that most of the axioms in \KIFDOLCE{} are defective.

For \CASLDOLCE{}, we have obtained 1,266 different falsity-tests to test 378 rules axioms. The results of testing \CASLDOLCE{} are summarized in Table \ref{subtable:TestingCASLDolce}. Only 30 falsity-tests (4.7\%) were proved. In all cases, the detected defects consist in trivial redundancies that are introduced by the automatic translation from CASL into TPTP syntax. In particular, the translation of \textPredicate{dj} (disjointness of universals) artificially introduces implications where its antecedent is redundant (detected by 20 different falsity-tests). 30 axioms were classified as defective, but all of them are correct although containing trivial redundancies. In addition, 1,093 truth-tests (86.33\%) were proved and 242 axioms are completely suitable (64.02\% of tested axioms) and 106 axioms as partially suitable (28.04\% of tested axioms). 324 different axioms are used in the proofs of those 1,093 truth-tests, which account for 77.88\% of total axioms. Consequently, we do not detect any substantial defect in \CASLDOLCE{} and nearly two thirds of the axioms are classified as completely suitable. 
Note also that most of the proofs were made in less than one second.
It is worth mentioning that \CASLDOLCE{} is a simplification of \KIFDOLCE{} and that they are not equivalent. For example, we cannot check whether conjecture (\ref{goal:world}) is entailed by \CASLDOLCE{} because the predicate \textPredicate{world} does not occur in it.

\subsection{Testing \KEPLER{}} \label{subsection:KEPLER}

The results of testing 5,753 of the \KEPLER{} axioms using 37,698 different falsity-tests and the same number of (unique) truth-tests are summarized in Table \ref{subtable:TestingKepler}. Of the 37,698 falsity-tests, 581 (1.54\%) were proved, enabling 610 axioms (10.60\%) to be classified as defective. Fortunately, all detected defects are either redundant subformulas or redundant axioms, such as the example at the end of Subsection \ref{subsec:redsubf}.
The level of redundancy is not surprising since the main objective of the Flyspeck project is to exhaustively check a complex mathematical proof. 
We conclude that not only is the number of defects detected in \KEPLER{} not high, but also that their nature is not critical. Indeed, only 336 different axioms are used in the proofs of the 581 proved falsity-tests.

\begin{threeparttable}[h!] \centering
\caption{  Experimentation results for \KEPLER{}}

\begin{subtable}{0.282\linewidth}
\begin{tabular} {|l|cc|cccc|cc|cc|}
\hline \\[-10pt]
 & \multicolumn{2}{c|}{Tests} & \multicolumn{4}{c|}{{Axioms}} & \multicolumn{2}{c|}{ Coverage} & \multicolumn{2}{c|}{Time}\\
& \multicolumn{1}{c}{\#Ts.} & \multicolumn{1}{c|}{Pr.} &  \multicolumn{1}{c}{{D.}} & \multicolumn{1}{c}{{C.S.}} & \multicolumn{1}{c}{{P.S.}} & \multicolumn{1}{c|}{{U.}} & \multicolumn{1}{c}{{\#Ax.}} & \multicolumn{1}{c|}{{Perc.}} & 
\multicolumn{1}{c}{$\leq$10} & \multicolumn{1}{c|}{$>$120}  \\
\hline \\[-10pt]
{ FT} & 37,698 & 581  & 610 & - & - & - & 336 & 0.43\% & 115 &  443\\
{ TT} & 37,698  & 22,825 & - & 649 & 4,298 & 196 & 2,969 & 3.78\% & 7,508 & 12,796\\
{ Total}  & 75,396 & 23,406 & 610 & 649 & 4,298 & 196  & 2,988 & 3.81\% & 7,623  & 13,239 \\
\hline
\end{tabular}
\caption{ \label{subtable:TestingKepler}  \KEPLER{}}
\end{subtable}

\begin{tablenotes}
\small

\item See table notes in Table \ref{table:Experimentation1}.
\end{tablenotes}
\label{table:Experimentation2}
\end{threeparttable}

\vspace{5mm}

With respect to suitability, 22,825 truth-tests (60.55\%) are proved, enabling us to decide that 649 axioms are completely suitable (11.28\% of tested axioms). 
Regarding the coverage measure, 2,988 formulas (3.81\% from the total of 78,500 formulas) are used in proofs and 5,753 formulas are tested (7.33\%). The small number of tested formulas is due to the fact that many formulas in the ontology do not contain any disjunctive connective, thus they are out of the scope of our technique. In other words, large parts of \KEPLER{} remain untested due to its syntactic form. We guess that a different (syntactically-based) test generation strategy could allow a larger coverage for \KEPLER{}.
It is worth mentioning that the set of axioms used in the proofs of truth- and falsity-tests are disjoint.
The execution times reflect that \KEPLER{} encodes a complex mathematical proof. In fact, 443 of the 581 proofs of falsity-tests, and 12,796 of the 22,825 proofs of truth-tests, take more than 120 seconds. 

\subsection{Testing and Improving \ADIMENSUMO{}} \label{subsection:AdimenSUMO}

In this subsection, we describe the process of testing and improving \ADIMENSUMO{} v2.4. Indeed, by correcting all the defects that were detected, we obtained \ADIMENSUMO{} v2.6, which was also evaluated.

We tested 1,622 axioms (58.24\% of rules) in \ADIMENSUMO{} v2.4 utilizing a set of 7,996 different falsity-tests. The results of this experiment  are summarized in Table \ref{subtable:TestingAdimenSUMOv2.4}. Although no defect was found by using the set of CQs proposed in \cite{ALR15}, most of the new falsity-tests (7,991) have been proved. 
Moreover, all the tested axioms were classified as defective, despite the ATPs proved the only 5 truth-tests that were built as the negation of the 5 non-proved falsity-tests).
The coverage for Adimen-SUMO v2.4 is only 1.35\%. This is due, likewise for \KIFDOLCE{}, to the fact that the proofs of the whole set of falsity-tests are based on a very small subset of axioms: 96 axioms (1.29\%). 
We correct every defective axiom, and apply  our testing methodology again to the resulting ontology, producing once more a new set of  defective axioms. Consequently, we have iteratively proceeded in this way until no new defect is found. This requires 9 iterations, after which we obtained \ADIMENSUMO{} v2.6. In total, we have corrected 21 typos, 3 redundant axioms, 17 incorrect axioms and 47 redundant subformulas (in axioms). In addition, 24 defective axioms --that were detected in preliminary experimentations-- have been corrected previously to this test.

\begin{threeparttable}[h!]
\caption{Experimentation results for \ADIMENSUMO{}}
\centering
\begin{subtable}{0.282\linewidth}
\begin{tabular} {|l|cc|cccc|cc|cc|}
\hline \\[-10pt]
 & \multicolumn{2}{c|}{Tests} & \multicolumn{4}{c|}{{Axioms}} & \multicolumn{2}{c|}{ Coverage} & \multicolumn{2}{c|}{Time}\\
& \multicolumn{1}{c}{\#Ts.} & \multicolumn{1}{c|}{Pr.} &  \multicolumn{1}{c}{{D.}} & \multicolumn{1}{c}{{CS.}} & \multicolumn{1}{c}{{P.S.}} & \multicolumn{1}{c|}{{U.}} & \multicolumn{1}{c}{{\#Ax.}} & \multicolumn{1}{c|}{{Perc.}} & 
\multicolumn{1}{c}{$\leq$10} & \multicolumn{1}{c|}{$>$120}  \\
\hline \\[-10pt]
{ FT} & 7,996 & 7,991 & 1,622 & - & - & - & 96 & 1.29\% & 2 & 7,955  \\
{ TT} & 7,996  & 5 & - & 0 & 0 & 0 & 14 & 0.19\% & 5 & 0\\
{ Total}  & 15,992 & 7,996 & 1,622 & 0 & 0 & 0 & 100 & 1.35\% & 7 & 7,955  \\
\hline
\end{tabular}
\caption{\label{subtable:TestingAdimenSUMOv2.4}\ADIMENSUMO{} v2.4}
\end{subtable}

\vspace{-3mm}
\begin{subtable}{0.282\linewidth}
\begin{tabular} {|l|cc|cccc|cc|cc|}
\hline \\[-10pt]
 & \multicolumn{2}{c|}{Tests} & \multicolumn{4}{c|}{{Axioms}} & \multicolumn{2}{c|}{ Coverage} & \multicolumn{2}{c|}{Time}\\
& \multicolumn{1}{c}{\#Ts.} & \multicolumn{1}{c|}{Pr.} &  \multicolumn{1}{c}{{D.}} & \multicolumn{1}{c}{{CS.}} & \multicolumn{1}{c}{{P.S.}} & \multicolumn{1}{c|}{{U.}} & \multicolumn{1}{c}{{\#Ax.}} & \multicolumn{1}{c|}{{Perc.}} & 
\multicolumn{1}{c}{$\leq$10} & \multicolumn{1}{c|}{$>$120}  \\
\hline \\[-10pt]
{ FT} & 8,010 & 0  & 0 & - & - & - & - & - & - & -  \\
{ TT} & 8,010  & 6,698 & - & 881 & 748 & 0 & 3,830 & 51,53\% & 6372 & 7\\
\hline
\end{tabular}
\caption{\label{subtable:TestingAdimenSUMOv2.6}\ADIMENSUMO{} v2.6}
\end{subtable}

\vspace{-3mm}
\begin{tablenotes}
\small
\item See table notes in Table \ref{table:Experimentation1}.
\end{tablenotes}
\label{table:Experimentation3} 
\end{threeparttable}

\vspace{5mm}

In \ADIMENSUMO{} v2.6, see Table \ref{subtable:TestingAdimenSUMOv2.6}, we tested 1,629 axioms (58.20\% of 2,799 rules) with a set of  8,010 different falsity-tests (and the same number of truth-tests). No falsity-test is proved, whereas 6,698 truth-tests (83.62\%) are proved. The latter result allows us to classify 881 axioms as being completely suitable (54.08\% of tested axioms) and 748 axioms as partially suitable (45.92\%). 
The coverage of Adimen-SUMO v2.6 is 51.53\%:  3,830 different axioms are used in the set of truth-test proofs. That coverage level is much more larger than the coverage we have ever accomplished using black-box testing techniques \cite{ALR17}. Thus, we can conclude that more than a half of \ADIMENSUMO{} v2.6 has been validated using the methodology introduced in this paper.

\section{Conclusions and Future Work} \label{section:conclusions}

This work offers a new practical insight towards the automatic testing of first-order logic (FOL) ontologies using ATPs. In addition, we provide formal proofs of the correctness of the proposed tests and report on the practical application of our methodology to four different FOL ontologies: \ADIMENSUMO{}, \KIFDOLCE{}, \CASLDOLCE{} and \KEPLER{}. Using our methodology, we have been able to detect some defects in all of those ontologies, although in the last two each defect can be trivially solved.
In the case of \ADIMENSUMO{}, we have applied our white-box testing approach to \ADIMENSUMO{} v2.4 and manually corrected all defects that have been detected. Following 9 iterations of the process of testing and correcting, we obtained \ADIMENSUMO{} v2.6, where our implementation does not detect any defect. Moreover, 54.32\% of the axioms in \ADIMENSUMO{} v2.6 were classified as completely suitable for reasoning purposes and 83.25\% of the tested axioms were classified as suitable. 

It is worth highlighting that testing the suitability for reasoning purposes of axioms is often much more interesting than checking the consistency of ontologies. For example, \KIFDOLCE{} is proved to be consistent by Vampire v4.1 although 99\% of its axioms are classified as defective and none is classified as suitable for reasoning purposes.  It is also particularly interesting to evaluate the competency of axioms by following black-box testing strategies, such as the one proposed in \cite{ALR15}.

All the resources that have been used and developed during this work are available in a single package, including:\footnote{The package is available at \url{http://adimen.si.ehu.es}.} a) the ontologies; b) tools for the creation of tests, the experimentation and the analysis of results; and c) the resulting tests for each ontology and the output obtained from different ATPs.

Regarding future work, ATPs are currently incorporating new techniques designed for reducing the large axiom space, such as the {\em axiom selection technique} \cite{SuD03,MeP09,HoV11,KuM16} and the {\em abstraction-refinement framework} \cite{TeW15,LoK17}.
Using these techniques we hope that a part of the test that --in the experimentation reported in this paper-- are unknown can be proved, hence we would find more defects in \ADIMENSUMO{} or classify more axioms as being completely suitable. This shall also allow us to evaluate the different techniques available and possible strategies for these techniques.
In addition, we could explore techniques to minimize the computational effort performed for testing based on automated reasoning algorithms
for determining the impact of axioms (as in \cite{NRG12}) and then perform the tests in decreasing order of the impact of the axioms that generate each test.

We have introduced a correct method of white-box test generation based on finding redundancies related to the disjunction (implication) logical operator. 
Other methods for automatic generation of tests from an axiomatization can be obtained by focusing in other kinds of redundancies. Thus, we plan to study new white-box testing strategies in the future. We also plan to apply our white-box testing methodology to other ontologies and to all the future versions of \ADIMENSUMO{}.

Finally, the proposed white-box testing opens an alternative way of checking the usefulness and non-defectiveness of axioms without proving the satisfiability of the whole ontology: in particular, by proposing tests that are specific to single axioms and by checking the use of axioms to prove the usefulness of other axioms.

\section*{Acknowledgements}
The authors are grateful to the anonymous reviewers for their valuable comments that improve this publication. 
This work has been partially funded by the Spanish Projects TUNER (TIN2015-65308-C5-1-R), COMMAS (TIN2013-46181-C2-2-R), and GRAMM (TIN2017-86727-C2-2-R) and the Basque Projects GIU15/30 and GIU18/182.

\bibliographystyle{plain}
\bibliography{white}

\end{document}

%% file: commands.tex
\newcommand{\function}[2]{#1_{/#2}}


\newcommand{\tab}{\hspace{20pt}}
\newcommand{\doubletab}{\tab\tab}
\newcommand{\connective}[1]{\bf #1 \;}
\newcommand{\predicate}[1]{\rm #1}
\newcommand{\constant}[1]{\rm #1}
\newcommand{\variable}[1]{\tt #1}
\newcommand{\rowVariable}[1]{\tt @#1}
\newcommand{\true}{\it true}
\newcommand{\false}{\it false}

\newcommand{\quantifier}[1]{\overline{Q#1}}

\newcommand{\nnf}[1]{nnf(#1)}
\newcommand{\negnnf}[1]{\widehat{#1}}


\newcommand{\textVariable}[1]{{\it{?#1}}}
\newcommand{\textConstant}[1]{{\it{#1}}}
\newcommand{\textFunction}[1]{{\it{#1}}}
\newcommand{\textPredicate}[1]{{\it{#1}}}
\newcommand{\metaConstant}[1]{{\it{\dollar #1}}}
\newcommand{\metaPredicate}[1]{{\dollar #1}}
\newcommand{\dollar}{\$}

\newcommand{\corrected}[1]{{\underline{#1}}}


\newcommand{\TT}[1]{TT(#1)}
\newcommand{\FT}[1]{FT(#1)}

\newcommand{\TTP}[1]{TT_0(#1)}
\newcommand{\FTP}[1]{FT_0(#1)}

\newcommand{\TTPP}[1]{TT_{ncf}(#1)}


\newcommand{\WORDNET}{WordNet}
\newcommand{\SUMO}{SUMO}
\newcommand{\CYC}{Cyc}
\newcommand{\DOLCE}{DOLCE}
\newcommand{\KIFDOLCE}{KIF-DOLCE}
\newcommand{\CASLDOLCE}{CASL-DOLCE}
\newcommand{\YAGO}{YAGO}
\newcommand{\TPTPSUMO}{TPTP-SUMO}
\newcommand{\ADIMENSUMO}{Adimen-SUMO}
\newcommand{\KEPLER}{FPK}
\newcommand{\DBPEDIA}{DBpedia}


\newtheorem{theorem}{Theorem}
\newtheorem{example}[theorem]{Example}
\newtheorem{lemma}[theorem]{Lemma}
\newtheorem{proposition}[theorem]{Proposition}
\newtheorem{definition}[theorem]{Definition}
\newtheorem{remark}[theorem]{Remark}